\documentclass{article}

\usepackage{xinstyle}
\usepackage{algorithm}
\usepackage{algpseudocode}
\usepackage[utf8]{inputenc} 
\usepackage[T1]{fontenc}    
\usepackage{hyperref}       
\usepackage{url}            
\usepackage{booktabs}       
\usepackage{amsfonts}       
\usepackage{nicefrac}       
\usepackage{microtype}      
\usepackage{xcolor}         
\usepackage{amsmath}
\usepackage{amssymb}
\usepackage{graphicx}
\usepackage{float}
\usepackage[english]{babel}
\usepackage{amsthm}
\usepackage{verbatim}

\newtheorem{theorem}{Theorem}[section]

\newtheorem{cor}{Corollary}[theorem]
\newtheorem{lemma}[theorem]{Lemma}
\newtheorem{thm}[theorem]{Theorem}

\newtheorem{pro}[theorem]{Proposition}
\newtheorem{aspt}{Assumption}[section]

\def\Prob{{\mathbb P}}
\def\E{{\mathbb E}}

\def\phat{\hat{p}}

\def\Otilde{\widetilde{O}}
\def\calM{\mathcal{M}}
\def\Zhat{\widehat{Z}}

\newcommand{\Prj}[1]{\Pi_{\mathcal{M}}(#1)}

\newcommand{\IDM}{\textnormal{IDM}}
\usepackage{subcaption}

\begin{document}

\title{Resolving memorization in empirical diffusion model for manifold data in high-dimensional spaces}
\author{
    Yang Lyu$^{\dagger}$,
    Tan Minh Nguyen$^{\S}$,
    Yuchun Qian$^{\ddagger}$, 
    Xin T. Tong$^{*}$\\
    \\
    Department of Mathematics, National University of Singapore
}

\date{}

\begingroup
\renewcommand\thefootnote{}
\footnotetext{$^{\dagger}$\texttt{yang.lyu@u.nus.edu}}
\footnotetext{$^{\S}$\texttt{tanmn@nus.edu.sg}}
\footnotetext{$^{\ddagger}$\texttt{yuchun\_qian@u.nus.edu}}
\footnotetext{$^{*}$Corresponding author. \texttt{xin.t.tong@nus.edu.sg}}
\endgroup

\maketitle

\begin{abstract}
Diffusion models are popular tools for generating new data samples, using a forward process that adds noise to data and a reverse process to denoise and produce samples. However, when the data distribution consists of $n$ points, empirical diffusion models tend to reproduce existing data points—a phenomenon known as the memorization effect. Current literature often addresses this with complex machine learning techniques. This work shows that the memorization issue can be solved simply by applying an inertia update at the end of the empirical diffusion simulation. Our inertial diffusion model requires only the empirical score function and no additional training. We demonstrate that the distribution of samples from this model approximates the true data distribution on a $C^2$ manifold of dimension $d$, within a Wasserstein-1 distance of $O(n^{-\frac{2}{d+4}})$. This bound significantly shrinks the Wasserstein distance between the population and empirical distributions, confirming that the inertial diffusion model produces new, diverse samples. Remarkably, this estimate is independent of the ambient space dimension, as no further training is needed. Our analysis leverages the fact that the inertial diffusion samples resemble Gaussian kernel density estimations on the manifold, revealing a novel connection between diffusion models and manifold learning.    
\end{abstract}


\section{Introduction}
Generative models are computational mechanisms trained to produce new synthetic data points from an unknown distribution where a limited amounts of training data is readily available. 
They have been implemented in various scenarios for the generation of images, audio and texts. 
Well known members of generative models include 
Generative Adversarial Networks (GANs), Variational Autoencoders (VAEs), and normalizing flows. Diffusion model (DM), also known as the denoising diffusion probabilistic model (DDPM)~\cite{ho2020denoising}, is a new addition to the family. It has gained great popularity due to its capability in generating high quality samples~\cite{dhariwal2021diffusion} and model reproducibility ~\cite{zhang2023emergence}. 

Mathematically, the idea behind DM can be explained using simple stochastic analysis~\cite{song2020score}.  Consider an Ornstein--Unlenbeck (OU) process starting from $X_0\sim p_0$, with 
\begin{equation}
\label{eqn:forward}
dX_t=-X_tdt+\sqrt{2}dW_t,     
\end{equation}
where $W_t$ is a $D$-dimensional Brownian motion.
The Fokker Plank equation (FPE) describes the distribution of $X_t$: $\tfrac{d}{dt}p_t=\nabla \cdot (x p_t)+\Delta p_t$. Due to the linearity of the OU process, we can write
\begin{equation}
\label{eqn:popdensity}
p_t(x)=\int K_{\sigma_t}(x, \alpha_t y)p_0(dy),    
\end{equation}
using $\alpha_t=e^{-t},\sigma^2_t=1-e^{-2t}$ along with Gaussian kernel 
\[
K_{\sigma}(x,y)=(2\pi \sigma^2)^{-D/2} \exp\left(-\frac{1}{2\sigma^2}\|x-y\|^2\right). 
\]
The reversed distribution $q_{t}=p_{T-t}$ follows 
\[
\frac{d}{dt}q_t=-\nabla \cdot (x q_t)-\Delta q_t=\nabla \cdot (-(x+\nabla \log q_t)q_t)=\nabla \cdot (-(x+2\nabla \log q_t)q_t)+\Delta q_t. 
\]
Now consider an ordinary differential equation (ODE)
\begin{equation}
\label{eqn:ODE}
\frac{d}{dt}Y_t=Y_t+\nabla \log p_{T-t}(Y_t). 
\end{equation}
or a stochastic differential equation (SDE)
\begin{equation}
\label{eqn:SDE}
d\widetilde{Y}_t=(\widetilde{Y}_t+2\nabla \log p_{T-t}(\widetilde{Y}_t)) dt+\sqrt{2}dB_t,
\end{equation}
where $B_t$ is another $D$-dimensional Brownian
Suppose $Y_0,\widetilde{Y}_0\sim q_{0}=p_T$, using continuity equation and FPE, it is straight forward to check that the distributions of $Y_t$ and $\widetilde{Y}_t$ follow $q_t=p_{T-t}$. In particular, if we simulate \eqref{eqn:ODE} or \eqref{eqn:SDE}, $Y_T$ and $\widetilde{Y}_T$ will be samples from $p_0$. For the simplicity of exposition, we will focus on the ODE version when discussion DM. We remark that most of our results also apply to the SDE version. 

\subsection{Empirical DM and memorization}
One significant challenge confronting all generative models is \emph{memorization}, where the generated data closely resembles—or occasionally replicates exactly—a training data point. Memorization not only undermines the goal of generating truly ``new" data points but can also lead to legal concerns related to plagiarism or privacy violations. Several studies ~\cite{pidstrigach2022score,li2024good, lu2023mathematical} attribute this phenomena to the empirical DM, which can be viewed a full capacity limit of deep neural network (NN) trained DM used in practice.

Empirical DM can be easily obtained by replacing $p_0$ with its empirical estimate in the DM derivation.
In practice, we cannot access the population distribution $p_0$ directly. Rather, we have $n$ i.i.d. samples $X_{(i)}\sim p_0$. We denote their empirical distribution as 
\[
\phat_0=\frac{1}{n}\sum_{i=1}^n \delta_{X_{(i)}},
\]
where $\delta_x$ stands for the Dirac measure at point $x$. Replacing $p_0$ with $\phat_0$ in \eqref{eqn:popdensity}, the empirical DM distribution is given by 
\begin{equation}
\label{eqn:emDMdist}
\phat_t(x)=\frac1n\sum_{j=1}^n K_{\sigma_t}(x, \alpha_t X_{(j)}).
\end{equation}
One can consider the DM ODE \eqref{eqn:ODE} with the empirical score, i.e. 
\begin{equation}
\label{eqn:emDM}
\frac{d}{dt}Z_t=Z_t+\nabla \log \phat_{T-t}(Z_t).    
\end{equation}
It is not difficult to show that if $Z_0$ is generated from $\phat_T$, the distribution of $Z_t$ will be the same as $\phat_{T-t}$. In particular, the distribution of the final outcome $Z_T$ will follow $\phat_0$, which is the training data distribution. 
While mathematically this is elegant, it also means $Z_T$ does not generate a new data point, but rather, it will be one of the training data point. 
This is often referred to as the \emph{memorization} phenomena.

In practice, $Z_0$ is usually generated from an approximation of $\phat_T$ when $T$ is very large. Due to the OU setup, this distribution is often conveniently as $\mathcal{N}(0, I_D)$. But this modification does not resolve much of the memorization issue. See Lemma \ref{lem:inipatch} for a detailed discussion. 

Empirical diffusion models (DM) are not typically implemented directly in practice. Instead, it is more common to train an DNN to approximate the score function. When the DNN reaches its full capacity, the empirical score can be viewed as the global minimizer of denoising score matching \cite{li2024good, lu2023mathematical}. In this context, memorization within the empirical DM can be interpreted as a fundamental cause. Therefore, addressing memorization in empirical DM is of critical importance.

\subsection{Generalization with inertia update} 
Since the discovery of memorization phenomena, there has been growing interest in finding ways to resolve it. One line of works  focus on the fact that DM with neural network implementation does not show memorization when the data size is rich. The main goal of this line is showing that if the empirical score $\nabla \log \phat_{T-t}$ is replaced by an appropriate neural network architecture approximation, memorization will not take effect \cite{oko2023diffusion,tang2024adaptivity,azangulov2024convergence}. Another line of works focus on modifying the empirical DM score function so $Z_T$ will not be any of the existing data point. Some examples include 
regularization schemes~\cite{baptista2025memorization}
and Wasserstein proximal operator~\cite{zhang2024wasserstein}.

By looking at these existing works, one may think that empirical score by itself is not enough to generate bonafide new data samples. One of the main goal of this work is to debunk this fallacy. 

We consider a simple adjustment to the empirical DM process \eqref{eqn:emDM} that consists of two step. In the first step, we generate the empirical DM process \eqref{eqn:emDM} up to time $T-h^2$. In the second step we update 
\begin{equation}
\label{eqn:iniup}
\Zhat_T=\alpha_{h^2}^{-1}(Z_{T-h^2}+\sigma^2_{h^2} \nabla \log \phat_{T-h^2}(Z_{T-h^2})). 
\end{equation}
Here $h$ is a small tuning parameter, which plays the role similar to the bandwidth in kernel based learning methods. We will refer \eqref{eqn:iniup} as an \emph{inertia} update, since it only uses the empirical score at time $T-h^2$ to generate an update for the final time period $[T-h^2,T]$. We will also refer the data generation process as \emph{inertia diffusion model} (IDM) in subsequent discussion. Its pseudo code can be found in Algorithm \ref{alg:iDM}. We name it IDM as its trajectory is similar to running the DM ODE \eqref{eqn:ODE} in the terminal time period $[T-h^2, T]$ with a fixed direction $\nabla \log \phat_{T-t}(Z_{T-t})$. A simple illustrate of the IDM procedure with $S^1$ being the latent manifold can be found in Figure \ref{fig:idm_illustration}. It is clear that inertia update pushes the samples onto the manifold and interpolates between the training data points.

\begin{algorithm}
\caption{Inertia diffusion model}\label{alg:iDM}
\begin{algorithmic}
\Require Empirical score $\nabla \log \phat_t(x), 0<t\leq T,$ bandwidth $h$. 
\State Sample $Z_0 \sim \phat_T$ or $\mathcal{N}(0,I_D)$
\State Solve ODE $\frac{d}{dt}Z_t=Z_t+\nabla \log \phat_{T-t}(Z_t)$ up to $T-h^2$. 
\State Update $\Zhat_T=\alpha^{-1}_{h^2}(Z_{T-h^2}+\sigma^2_{h^2} \nabla \log \phat_{T-h^2}(Z_{T-h^2}))$.
\State Output $\Zhat_T$ as a sample. 
\end{algorithmic}
\end{algorithm}

\begin{figure}[htbp]
    \centering
    \includegraphics[width = 0.6\textwidth]{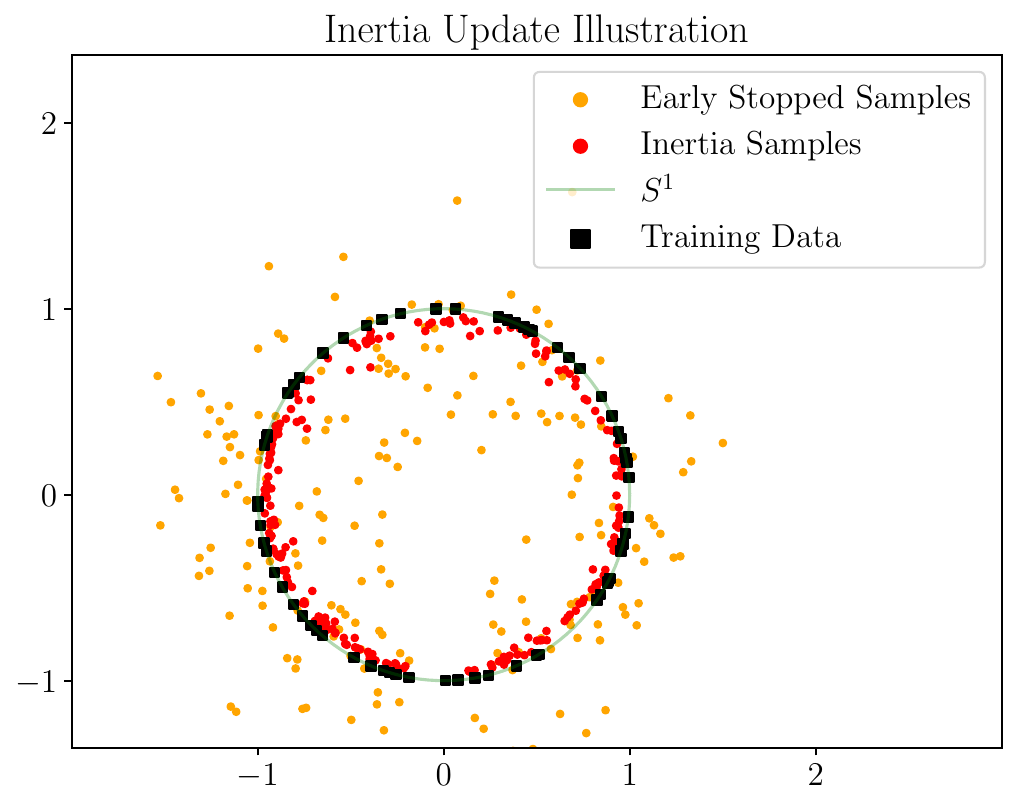}
    \caption{Illustration of the inertia update. 
    Training data (black boxes)  are $70$ i.i.d. samples uniformly drawn from $S^1$. 
    $200$ early stopped samples (yellow dots) are generated using Equation~\eqref{eqn:emDMdist} with $t=0.05$, which are then updated using Equation~\eqref{eqn:iniup} to IDM samples (red dots). 
    } \label{fig:idm_illustration}
\end{figure}

Comparing with other existing methods, the implementation of IDM requires only the empirical score function and no further training. It involves no requirement of DNN architecture specification or training modification like regularization. In particular, analysis of IDM does not require the assumption that the training loss optimization--which is often nonconvex--can be solved.

On the other hand, the training-free IDM does not leverage amortization like other DNN-based diffusion models. Specifically, generating a new data sample requires processing the entire training dataset, which can be computationally expensive in practice. To improve efficiency, a more practical implementation of IDM would involve replacing the empirical score $\hat{p}_t$ with an accurate approximation $\tilde{p}_t$, achieved using standard DNNs or other machine learning architectures. The primary advantage of IDM is that it eliminates concerns about the DM memorizing training data at full capacity.

To rigorously illustrate that IDM's capability on generalization, we consider the popular manifold data assumption, which has widely been seen as the reason why many machine learning methods can beat the curse of dimensionality. The detailed definition can be found in Section 2.  

\begin{theorem}
\label{thm:main1}[Informal]
Suppose $p_0$ is a regular $C^2$ density supported on a $C^2$ $d$-dimensional compact manifold $\calM$.  Suppose $X_{(1)},\ldots, X_{(n)}$ are i.i.d. samples from $p_0$. Choose bandwidth $h=n^{-\frac{1}{d+4}}$, let $\phat_{\text{IDM}}$ be the output distribution of Algorithm \ref{alg:iDM} with $Z_0\sim \phat_T$. Then with high probability 
\[
W_1(\phat_{\text{IDM}},p_0)\leq \Otilde(h^2). 
\]    
Here $W_1$ is the Wasserstein-1 distance, and $\widetilde{O}$ hides some polynomials of $\log n, d$ and information of $\calM$. But it is independent of $D$. 
\end{theorem}

To see why Theorem \ref{thm:main1} indicates IDM generates a new sample from $p_0$, we recall that with high probability $W_1(\phat_0,p_0)=\Otilde(n^{-\frac1d})$ \cite{fournier2015rate}, which we will referred to as the \emph{memorization} learning rate. So if $d\geq 5$, then $W_1(\phat_{\text{IDM}},p_0)\ll W_1(\phat_0,p_0)$, which indicates IDM has significantly better generalization capability than pure memorization. 

Another important remark is that the upper bound in Theorem \ref{thm:main1} is completely independent of $D$ and the sample size does need to surpass $D$. This is different from other existing DM generalization analysis.  

Finally, our work has a close connection with manifold learning literature. In particular  1) we can show the inertia update have high probability "projecting" points onto $\calM$ approximately; 2) we can also show $W_1(\phat_{\text{IDM}},\phat_{\text{KDE}})=\Otilde(h^2)$ where $\phat_{\text{KDE}}$ is the kernel density estimation (KDE) on $\calM$ with Gaussian kernel bandwidth $h$. We delay the detailed discussion to Section \ref{sec:method}.

\subsection{Literature review}
Since the invention of DM \cite{ho2020denoising,song2020score}, it has gained great interest in the community. There have been many works aiming to understand its performance \cite{chen2022sampling} or designing  modifications for various purposes \cite{graikos2022diffusion}. It is unlikely that we can cover every line of this fast growing literature, and we will only discuss works that are  relevant to the memorization phenomena. 

One line of works  focus on statistical learning rate of DM for the estimation of the data distribution $p_0$. Earlier works like \cite{chen2022sampling} indicates that the accuracy of DM depends on the accuracy of the population score $\nabla \log p_t$ estimation. \cite{oko2023diffusion} indicates the indicates the learning rate of deep neural network (DNN) based DM is  $\Otilde(C_D n^{-\frac{s+1}{2s+D}})$ in $W_1$ distance, where $s$ is the smoothness of the density $p_0$ and $C_D$ is a constant that depends on $D$. While this rate is minimax optimal, it is also subject to the curse of ambient dimension $D$, which can scale up to millions for practical problems. 

To explain why DM in practice can have good performance despite the curse of dimension, several works have focused on the popular manifold hypothesis~\cite{fefferman2016testing}, where the data is assumed to lie on a $d$-dimensional manifold  $\calM$ with smoothness $\beta$. $d$ is assumed to be significantly smaller than $D$ and often viewed as a $O(1)$ constant. Recently, 
\cite{tang2024adaptivity} shows that with appropriate DNN structure, DM can achieve $\Otilde(D^{s+d/2}\max\{n^{-\frac{s+1}{2s+d}},n^{-\frac{\beta}{2s+d}}\})$ rate. Later \cite{azangulov2024convergence} improves estimate to $\Otilde(\sqrt{D}n^{-\frac{s+1}{2s+d}})$ but the manifold smoothness $\beta$ needs to be sufficiently high and the sample size $n\gtrsim D$. Both these rates are better than the memorization rate $W_1(\phat_0, p_0)=\Otilde(n^{-\frac{1}{d}})$ \cite{fournier2015rate}, so they indicate DM can generate new data under the manifold hypothesis.
Meanwhile, 
both these works leverage on a specific DNN structure to approximate the population score (see equation (15) in \cite{azangulov2024convergence}), and this structure need to be configured according to the training data. In other words, implementation such DNN structure can be impractical. 
Besides manifold hypothesis, one can also leverage the locality of the distribution to resolve memorization and improve estimation accuracy \cite{kamb2024analytic,gottwald2025localized}. But this is a different framework and the results are difficult to compare with ours.

Another line of works focus on modifying the empirical DM score function so $Z_T$ will not be any of the existing data point. Some examples include 
regularization schemes~\cite{baptista2025memorization}, Wasserstein proximal operator~\cite{zhang2024wasserstein}, and local manifold dimension fitting ~\cite{ross2024geometric}
. While these works demonstrate convincing empirical evidence, there has been no rigorous analysis showing they are better estimator of $p_0$ than pure memorization.

One very interesting connection this paper create is the link between DM and manifold learning. In particular, as one important step, we show that IDM will generate samples that are approximately following the kernel density estimator (KDE) on $\calM$. Manifold KDE is known to have $O(n^{-\frac{2}{d+4}})$ rate for $C^2$ manifold and $p_0$ \cite{ozakin2009submanifold}. But to implement it, one needs information of $\calM$, e.g. a proper parameterization on $\calM$. In particular, manifold KDE does not provide a direct way to generate new samples from $p_0$. On the other hand, there is also a big and constantly growing literature on inference of manifold \cite{Roweis2000,BelkinNiyogi2003,coifman2006diffusion,Singer2006,khoo2024temporal}. Many of these methods consider estimation of the manifold Laplace-Beltrami operator ~\cite{hsu2002stochastic}  or a linear interpolation matrix. The implementation of these methods often require processing pairwise distance between training data points, which is often $O(n^2)$ computational complexity. This can be prohibitive in applications with large sample size.

Manifold kernel density estimation, which also relies on approximating and estimating \(\Delta_{\mathcal{M}}\) but directly provides distribution estimation on manifold data, often incorporates curvature corrections to estimate distributions more accurately \cite{Pelletier2005,HenryRodriguez2009,KimPark2013}.\cite{divol2022measure} proves a curvature-corrected KDE attains the minimax optimal Wasserstein rate in intrinsic dimension \(d\) \cite{niles2022minimax}. \cite{divol2021minimax} employs a sampler supported on the convex hull of the data points akin to ours but targets manifold reconstruction and relies on a  graph constructions. \cite{scarvelis2023closed} proposed an empirical score based barycentric prior that forces generated points into barycenters of subsets of data points, while empirically successful, lack theoretical guarantees. 
Moreover, as mentioned above, these methods have at least \(O(n^2)\) computational complexity, and some require knowledge of \(d\) or oracle access the the manifold.

\subsection{Our contribution}
This paper has proposed and analyzed  the IDM approach. 
We showed rigorously  that IDM avoid the memorization issue. Moreover, comparing with the existing literature, IDM has the following advantage:
\begin{enumerate}
    \item IDM is a simple adaptation of the existing empirical DM. It uses only the empirical score function. It does not require any additional training procedure or configuration. As a consequence, its learning rate under the manifold assumption, $\Otilde_d(n^{-\frac{2}{4+d}})$ is independent of ambient dimension $D$. Also this rate can be reached with theoretical guarantee, while others methods' learning rate can only be  achieved if the non-convex optimization problem related to model training is perfectly solved. 
    Finally, understanding of IDM does not rely on any specific machine learning architecture, so it will not be subject to the fast update cycle of machine learning literature. 
    \item Comparing with the distribution learning literature, our case corresponds to the $s=\beta=2$ case, and our exponent, $n^{-\frac{2}{4+d}}$, matches the one from \cite{tang2024adaptivity}. Admittedly, when the manifold and density smoothness are higher than $2$, \cite{azangulov2024convergence} has an improved exponent $\Otilde(\sqrt{D} n^{-\frac{s+1}{2s+d}})$. This is because IDM matches the performance of a Gaussian KDE, and Gaussian kernel cannot leverage smoothness beyond second order \cite{tsybakov2009nonparametric}. On the other hand, if we consider $s=2$ while letting $\beta$ to be sufficiently large, in order for $\sqrt{D} n^{-\frac{3}{4+d}}\leq n^{-\frac{2}{4+d}}$, one would need $n\geq D^{4+d}$, which is often a requirement cannot be met even with moderate dimension settings like $D=10^3$ and $d=1$. Moreover, even if there are very high level of smoothness, estimates in \cite{azangulov2024convergence} require $n\geq D$, which can be a prohibitive condition for very large scale models. In comparison, IDM does not have such constraint.
    \item Under the manifold hypothesis, IDM is closely related to the manifold KDE, a density estimation tool in the manifold learning literature. This reveals an interesting connection between diffusion models and manifold learning. Such connection has not been discussed in the literature. 
    \item Comparing to manifold KDE, implementation of IDM does not require knowledge of the underlying manifold. 
    While manifold learning methods like graphical Laplacian can be used to infer the manifold, the computational cost usually scale like $O(n^2)$. Computational cost IDM is only $O(n)$. 
    \item As a side product, we show the IDM update is a Naradaya--Watson estimator, and it is a generalized version of manifold projection for perturbed data points. This can be an interesting finding for other purposes.
\end{enumerate}

\subsection{Notation and convention}
Throughout the paper, we assume the data lies on a $d$-dimensional compact Riemannian manifold $\calM$ embedded in $\mathbb{R}^D$. We treat $d$ and all constants related to $\calM$ as order $1$ constants. We keep track of all quantity relate to the ambient space dimension $D$ and the sample size $n$, which both can diverge to $\infty$ with different speed. If a quantity that does not depend on $D$ nor $n$, we refer to it as a \emph{constant}. 

We say an event happens with high probability if it happens with probability $1-n^{-k}$. Here $k$ is fixed constant that is independent of $D$ and $n$.

We say $f(D,n)\lesssim g(D,n)$ or $f(D,n)=O(g(D,n))$ if there is a constant $C$ that does not depend on $D$ or $n$ so that $f(D,n)\leq Cg(D,n)$. But this constant may depend on configuration of $\calM$. We write $f(D,n)=\widetilde{O}(g(D,n))$, if there is a constant $C$ that does not depend on $D$ or $n$ so that $f(D,n)\leq C(k\log n )^{\alpha}g(D,n)$  for certain $\alpha>0$.

A chart of a manifold $\mathcal{M}$ is a pair $(U, \phi)$ where $U \subset \mathcal{M}$ is an open set, and $\phi: U \rightarrow \mathbb{R}^d$ is a homeomorphism, i.e., $\phi$ is bijective and both $\phi$ and $\phi^{-1}$ are continuous.
We say that a manifold \( \mathcal{M} \) is \( C^2 \) if for every \( y \in \mathcal{M} \), there exists an open set \( U_y \subset \mathbb{R}^d \) and a one-to-one map \( \Phi_y: U_y \to \calM\subset\mathbb{R}^D \) such that \( \Phi_y(0) = y \) and \( \Phi_y \) is \( C^2 \)-smooth. In this case, \( \mathcal{M} \) can be locally represented as the image of a \( C^2 \) function. 
Given a point \( p \in \mathcal{M} \), the exponential map \( \exp_p: T_p\mathcal{M} \to \mathcal{M} \) is defined by mapping a tangent vector \( v \in T_p\mathcal{M} \) to the point reached at time \( t=1 \) along the unique geodesic starting at \( p \) with initial velocity \( v \).
 We say that \(\mu\) has a \(C^2\) density, if its density \(p\) is \(C^2\), i.e., for every chart\(  (U, \phi)\) of \( \mathcal{M}\), the function \(p \circ \phi^{-1}: \phi(U) \rightarrow \mathbb{R} \text { is a } C^2 \text { function. }\)
\section{Method and analysis}
\label{sec:method}
\subsection{Derivation of IDM}
There has been no clear definition or metric deciding if a given generative model is truly generative. 
Different existing works use different metrics \cite{zhang2023emergence,zhang2024wasserstein,baptista2025memorization}.

The manifold hypothesis \cite{fefferman2016testing} is a common setup used in the literature to explain why machine learning models avoid the curse of dimension in practice. It assumes the data points lie on a latent low dimensional sub-manifold:   
\begin{aspt}
\label{aspt:manifold}
The following hold
\begin{enumerate}
\item $\calM$ is a compact $C^2$ $d$-dimensional Riemannian manifold without boundary.
\item $\calM$ is locally isometrically embedded in $\mathbb{R}^D$. In particular, there are some constant $C$ so that 
\[
|\|x-y\|-\text{dist}_{\calM}(x,y)|\leq C\text{dist}_{\calM}(x,y)^3.
\]
\item The data distribution $p_0$ is supported on $\calM$ with a $C^2$ density $p_0(x)$.
\item The density is bounded from both sides $p_{\min}\leq p_0(x)\leq p_{\max}$.
\item The data points $X_{(i)}$ are i.i.d. samples from $p_0$. 
\end{enumerate}
\end{aspt}
A similar version of this assumption can be found in \cite{azangulov2024convergence}.

Given the manifold data assumption, intuitively, a truly generative model should produce outcomes that are 1) Close to the manifold $\calM$, since $\calM$ captures latent features within data. 2) Different from existing data points, otherwise it is purely memorization. Each of two criteria is easy to fulfill by itself. $\phat_0$ will automatically be supported on $\calM$, but it is pure memorization. Intuitively, we can consider perturbing each data point with Gaussian noise  $N(0,\delta I_D)$ so the samples are "new" (This distribution is close to $\phat_{\delta}$). But most of these samples will be of distance $\sqrt{D\delta}$ from $\calM$, which can be significant deviation if $D$ is large.  

Given these observations, it is natural to consider a two step procedure: We first generate random samples different from $X_{[n]}$, say using $\phat_{\delta}$, then push these points onto $\calM$. Then intuitively both basic criteria above will satisfy. In the data driven setting, we do not know the exact location and shape of $\calM$,  so finding the push map is the main challenge. However, we note that the DM dynamics will automatically push $p_t$, which is a distribution off $\calM$, onto $\calM$. So it is natural to ask how does it achieve this.

Given a manifold $\calM$, we say an  $x\in \mathbb{R}^D$ is uniquely normal projected (UNP), if $\Pi_{\calM}(x)\in \calM$ is the unique minimizer of $\|x-y\|$ with $y\in \calM$.
It is well known that the $\tau$ neighborhoood $\calM^\tau=\{x:\text{dist}(x,\calM)\leq \tau \}$  has all point being UNP if $\tau$ is below some threshold \cite{leobacher2021existence}. The  threshold $\tau_0$ is often referred as the reach of $\calM$. 
Given these concept, we have the following characterization of the population score function, of which the asymptotic version appeared in \cite{lu2023mathematical}.   
\begin{pro}
\label{pro:popDM}
Suppose Assumption \ref{aspt:manifold}  holds and $x\in \calM^{\tau}$ with a sufficiently small $\tau$, there is a threshold $t_0$ such that if $h^2<t_0$
\[
\nabla \log p_{h^2}(x) =
\frac{\Pi_{\alpha_{h^2}\calM}(x)-x}{\sigma_{h^2}^2}+
O(1+dist(x,M)^{3/2}/h^2+\sqrt{d\log (1/h)}/h). 
\]
\end{pro}
The proof of Proposition \ref{pro:popDM} is allocated in Section \ref{sec:popDM}. Given Proposition \ref{pro:popDM}, we see that the population inertial map 
\begin{equation}
F_{\sigma_{h^2}}(x)=\alpha_{h^2}^{-1}(x+\sigma^2_{h^2} \nabla \log p_{h^2}(x))
\end{equation}
is a $\Otilde(\delta)$ approximation of $\Pi_\calM(\alpha_\delta x)$ and $\Pi_\calM(x)$. From this, it is natural to consider using the empirical version of inertia map as a projection to the manifold. This will give us the inertia diffusion model.

\subsection{Inertia update as manifold projection}

Next, we investigate how well does the empirical inertia update inherit the projection property.  Note that the empirical score can be  explicitly written as 
\[
\nabla_x \log \phat_{h^2}(x)
=\frac{\sum_{i=1}^n \nabla_x K_{\sigma_{h^2}}(\alpha_{h^2}X_{(i)},x)}{\sum_{i=1}^n K_{\sigma_{h^2}}(\alpha_{h^2}X_{(i)},x)}
=-\frac{1}{\sigma_{h^2}^2} \sum_{i=1}^n w_i(x)\left(x-\alpha_{h^2} X_{(i)}\right),
\]
where $w_i(x)=\frac{K_{\sigma_{h^2}}(\alpha_{h^2}X_{(i)},x)}{\sum_{j=1}^n K_{\sigma_{h^2}}(\alpha_{h^2}X_{(i)},x)}$. To continue, given a data set $X=\{X_{(1)},\ldots,X_{(n)}\}$ and a bandwidth $\sigma$, we define the Nadaraya–Watson  estimator (NWE):
\begin{equation}
\label{eqn:NWE}
F_{X, \sigma}(z):= \frac{\sum\limits_{i=1}^N K_\sigma(z, X_{(i)}) X_{(i)}}{\sum\limits_{i=1}^N K_\sigma(z, X_{(i)})}.
\end{equation}
Such estimator is commonly used in kernel based nonparametric estimation.

Note that samples from $\phat_{h^2}$ can be written $z=\alpha_{h^2} X_{(i)}+\sigma_{h^2}\xi$ with some data point $X_{(i)}$ and a Gaussian random vector $\xi\sim \mathcal{N}(0,I_D)$. Its image with the empirical inertia update is given by 
\[
\alpha^{-1}_{h^2}(z+\sigma_{h^2}^2\nabla \log \phat_{h^2}(z))=\alpha^{-1}_{h^2} F_{\alpha_{h^2} X, \sigma_{h^2}}(\alpha_{h^2} X_{(i)}+\sigma_{h^2}\xi)=F_{X,\sigma'}(X_{(i)}+\sigma'\xi),
\]
with $\sigma'=\frac{\sigma_{h^2}}{\alpha_{h^2}}=h(1+O(h^2))$. 
If we denote the projection of $\xi$ on tangent space $\mathcal{T}_{X_{(i)}}\calM$ as $\xi_{\mathcal{T}}$, and the exponential map on $\calM$ as $\exp_{X_{(i)}}(v)$ with any element  $v\in \mathcal{T}_{X_{(i)}}$. The following theorem characterize the difference between the exponential map and IDM update. 
\begin{thm}
\label{thm:main2}
Under Assumption \ref{aspt:manifold}, let $\sigma\asymp n^{-\frac{4}{d+4}}$. Fix  $x\in \calM$ and $r>0$. Let $\xi\sim \mathcal{N}(0,I)$ with decomposition $\xi=\xi_{\mathcal{T}}+\xi_\bot$ on $T_x\calM\oplus N_x\calM$, then 
\begin{equation}
\label{eqn:NWproj2}
\mathbb{P}\left( \left\| F_{X,\sigma}(x + \sigma \xi) - \exp_x(\sigma \xi_{\mathcal{T}})\right\| \lesssim (d(k+1)\log n\, \sigma)^2\right) \ge 1 - n^{-k}.
\end{equation}
In particular, 
\begin{equation}
\label{eqn:NWproj2}
\mathbb{P}\left( \left\| F_{X,\sigma}(x + \sigma \xi_\bot) - x\right\| \lesssim (d(k+1)\log n\, \sigma)^2\right) \ge 1 - n^{-k}.
\end{equation}
\end{thm}
The proof is allocated in the end of Section \ref{sec:linftyproof}.
In view of this theorem, the NW estimator effectively truncates the noisy components in the normal bundle, and move along the geodesic of $\calM$ along the tangent space component. 

From \eqref{eqn:NWproj2}, we see $F_{X,\sigma}$ can be viewed as a  high probability projection of points outside of the manifold onto the manifold. This is an interesting phenomena, since the typical distance from $x+\sigma\xi_\bot$ to $\calM$ is $O(\sqrt{D}\sigma)$. When $D$ is very large, $\sqrt{D}\sigma$ can be larger than the "reach" of the manifold, so one cannot easily use standard differential geometric tools to establish this result.  

\subsection{IDM as manifold Kernel density estimation}
Kernel density estimation (KDE) is a classical method for non parametric density estimation. Given $n$ data points $\{X_{(i)}\}_{i\in [n]}$, the $d$-dimensional KDE density with Gaussian bandwidth $\sigma$ is given by 
\[
\phat_{\text{KDE}}(x)=\frac1n\sum\limits_{i=1}^n k_\sigma(x, X_{(i)}),\quad 
k_{\sigma}(x,y)=(2\pi \sigma^2)^{-d/2} \exp\left(-\frac{1}{2\sigma^2}\|x-y\|^2\right).
\] 
Note that we use $k_\sigma$ to denote the $d$-dimensional Gaussian kernel, which is different from the $D$-dimensional Gaussian kernel $K_\sigma$ by a factor of $(2\pi \sigma^2)^{-(D-d)/2}$.

When the manifold $\calM$ and its dimension $d$ are known to us, $\phat_{\text{KDE}}(x)$ defines a density for $x\in \calM$. It is well known that   $\phat_{\text{KDE}}(x)$ is an $O(n^{-\frac{2}{d+4}})$  approximation. The $L_2$ error, i.e. mean square error, is analyzed by \cite{ozakin2009submanifold}. We can improve it to $L_\infty$ error:
\begin{pro}
\label{prop:KDEp0}
Suppose Assumption \ref{aspt:manifold}  holds and $X_{(1)},\ldots, X_{(n)}$ are i.i.d. samples from $p_0$. 
If we choose $\sigma \asymp n^{-\frac{1}{d+4}}$ in the KDE estimator, with high probability $1-n^{-k}$
\[
\sup_{x\in \calM}\|\phat_{\text{KDE}}(x)-p_0(x)\|=O(n^{-\frac{2}{d+4}}\sqrt{(d+k)\log n}).
\]
\end{pro}
The proof is allocated in Section \ref{sec:linftyproof}.
Notably, the $O(n^{-\frac{2}{d+4}})$ rate is known as the minimax $C^2$-density estimation rate under our smoothness assumption \cite{tsybakov2009nonparametric}. Using optimal transport inequality, Proposition 7.10 of \cite {villani2021topics}, we have 
\[
W_1(\phat_{\text{KDE}}(x),p_0(x))\lesssim \|\phat_{\text{KDE}}(x)-p_0(x)\|_{L_\infty}=\Otilde(n^{-\frac{2}{d+4}}) 
\]
Note also in our context, this indicates that $\phat_{\text{KDE}}(x)$ is also generative.

However, in order to use $\phat_{\text{KDE}}(x)$, we need to know information of $\calM$, which is often difficult to derive in the practical data driven setting. In view of Theorem \ref{thm:main2}, we see that samples from $\phat_{\text{IDM}}$ can be roughly seen as $\exp_{X_U}(\sigma \xi_{\mathcal{T}})$, where $U$ is an index  randomly selected in $[n]$ by the uniform distribution and $\xi_{\mathcal{T}}$ is a standard $d$-dimensional Gaussian random vectors in the tangent space. 

When $\calM$ is flat, i.e. a $d$ dimensional subspace, it's easy to see that $\exp_{X_U}(\sigma \xi_{\mathcal{T}})=X_U+\sigma \xi_{\mathcal{T}}$ has exactly the same distribution as $\phat_{\text{KDE}}$. For general $\calM$, the two distributions are off by $\Otilde(\sigma^2)$ in $W_1$ distance:

\begin{pro}[Expectation Representation of Manifold KDE Integral]
\label{pro:KDEexp}
For sufficiently small \( \sigma \), given any $X=[X_{(1)},\ldots, X_{(n)}]$ and Lip-1 $f$
\[
\left|\int_{\mathcal{M}} \hat{p}_{\mathrm{KDE}}(x) f(x) \, dV(x) - \mathbb{E}_{U, \xi_{\mathcal{T}}}[ f\left( \exp_{X_U}(\sigma \xi_{\mathcal{T}}) \right)] \right|\leq (2(d+4)\log (2\pi/\sigma))^{\frac{3}{2}}\sigma^2.
\]
\end{pro}
The proof is allocated in Section \ref{sec:KDEproof}.

\subsection{Conclusion }
Combining Theorem \ref{thm:main2}, Propositions \ref{pro:KDEexp} and \ref{prop:KDEp0}, it is quite straight forward  to obtain the detailed version of 
Theorem \ref{thm:main1}
\begin{theorem}
\label{thm:main3}
Under Assumption \ref{aspt:manifold}. Choose bandwidth $h=n^{-\frac{1}{d+4}}$, let $\phat_{\text{IDM}}$ be the output distribution of Algorithm \ref{alg:iDM} with $Z_0\sim \phat_T$, or $Z_0\sim \mathcal{N}(0,I_D)$ with $T>C(\log (Dnk))$. 
There is a threshold $n_0$ independent of $D$. If $n>n_0$,
then with high probability $1-n^{-k}$
\[
W_1(\phat_{\text{IDM}},p_0)\leq O( d^2 k^2 n^{-\frac{2}{d+4}}). 
\]    
Here $W_1$ is the Wasserstein-1 distance, and $O$ hides some constants that depend on $\calM$. 
\end{theorem}
The proof is allocated in Section \ref{sec:proofmain}.

\def\CZeroValue{0.8}
\def\CZeroValueSecond{0.8}
\def\EmpiricalIDMRate{-0.201}
\def\EmpiricalMemorizedDMRate{-0.168}

\section{Numerical Experiments}

In this section, we numerically verify behaviors of data generation error $W_1(\hat{p}_{\text{IDM}}, p_0)$ using synthetic data generated from distributions supported on hidden low-dimensional manifolds embedded in (moderately) high-dimensional ambient spaces. 

To evaluate $W_1(\hat{p}_{\text{IDM}}, p_0)$, we use $M = 10^{6}$ i.i.d. samples from both distributions as proxies and measure the $W_1$ distance between these empirical measures. 
We approximate this distance with the debiased Sinkhorn divergence using \texttt{Geomloss}~\cite{feydy2019interpolating} package with regularization parameter $\epsilon = 10^{-3}$ and set the \texttt{scaling} parameter, which balances the trade-off between accuracy and computational cost, to $0.9$. This setting ensures a relatively accurate approximation of $W_1$ distances within a reasonable computational budget.

\begin{figure}[htbp]
  \centering
  \includegraphics[width=1.0\textwidth]{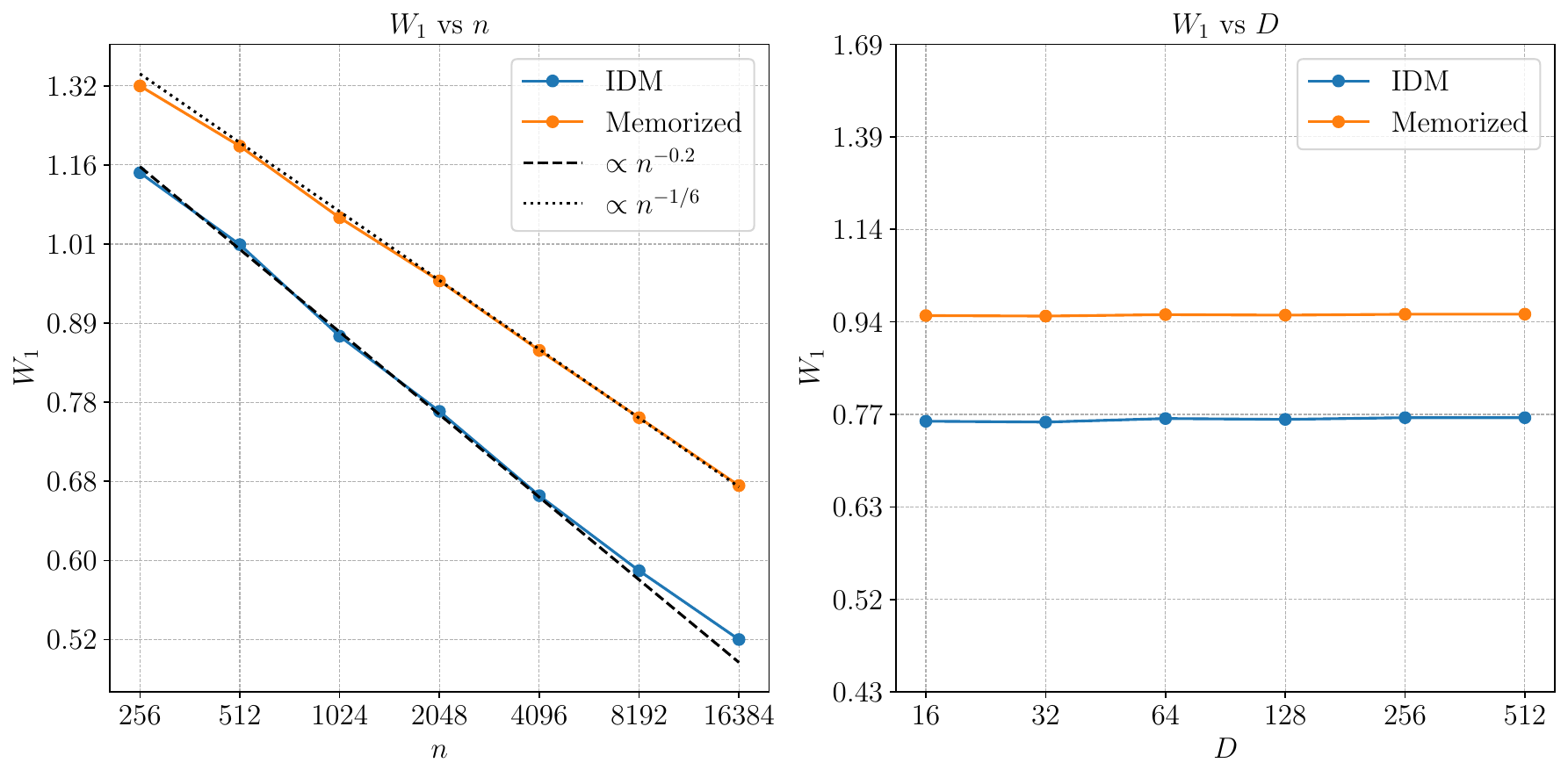}
  \caption{$W_1$ error of IDM vs memorized DM for data generated from the Haar measure on $\text{SO}(4)$ embedded in $\mathbb{R}^D$. Both plots are in log-log scale. (Left) $W_1$ loss w.r.t. sample size $n$ with $D=50$. Theoretical predictions for the algorithms are shown as dotted and dashed reference lines. (Right) $W_1$ error w.r.t. embedding dimension $D$ with fixed $n=2048$.}
  \label{fig:exp:combined}
\end{figure}

\subsection{Generation performance with different sample size}

We first validate that IDM can attain the $\tilde{O}_{d}(n^{-\frac{2}{d+4}})$ convergence rate for $W_1(\hat{p}_{\IDM}, p_0)$, which is the main claim of this paper.

We set the ground truth data to be generated from the Haar measure on the special orthogonal group $\text{SO}(4)$, a compact Riemannian manifold with intrinsic dimension $d = 6$ and naturally embedded in $\mathbb{R}^{4 \times 4}$, using algorithm described in~\cite{mezzadri2006generate} and provided by the \texttt{scipy} package.
This $16$-dimensional data was then embedded into $\mathbb{R}^D$ for $D = 50$ using a random orthogonal matrix.

We then measured $W_1(\hat{p}_{\IDM}, p_0)$ for $n$ ranging from $2^8$ to $2^{14}$ in powers of $2$. 
The IDM algorithm was configured with a bandwidth scheme $\displaystyle \frac{\sigma_{h^2}}{\alpha_{h^2}} =  \sigma = C_0 n^{-\frac{1}{d+4}}$ for $C_0 = \CZeroValue$. More specifically, the IDM samples were first generated from Equation~\eqref{eqn:emDMdist} with $t=h^2$, then updated using Equation~\eqref{eqn:iniup} with the same $h^2$ value.
In practice one may want to tune the value of $C_0$ to obtain an optimal performance on different problems.
As a baseline comparison, we also measured $W_1(\hat{p}_{\text{0}}, p_0)$, where $\hat{p}_{\text{0}}$ is the memorized version of DM, simulated with the empirical distribution $\phat_0=\frac{1}{n}\sum_{i=1}^n \delta_{X_{(i)}}$  which is known to converge in a rate of $\tilde{O}_{d}(n^{-1/d})$~\cite{weed2019sharp}.

The left panel of Figure~\ref{fig:exp:combined} plots the estimated $W_1$ distances against the sample size $n$ on a log-log scale. To verify the theoretical convergence rates, we include reference lines with slopes corresponding to the predicted exponents. For IDM, the theoretical rate is $\tilde{O}_{d}(n^{-\frac{2}{d+4}})$, which for $d=6$ corresponds to a slope of $-2/(d+4) = -0.2$. For the memorized DM, the theoretical rate is $\tilde{O}_{d}(n^{-\frac{1}{d}})$, corresponding to a slope of $-1/d \approx -0.167$. The empirical results for both methods align well with their respective theoretical prediction lines, demonstrating the expected power-law dependence. We observe a slight deviation for the largest two values of $n$ for IDM, which is attributable to the error in the $W_1$ distance estimation, as $n$ becomes comparable to the number of proxy samples $M$.
These results empirically validate the prediction of Theorem~\ref{thm:main3}.

\subsection{Generation performance with different embedding dimension}

We further investigate the influence of the ambient embedding dimension $D$ on $W_1(\hat{p}_{\IDM}, p_0)$.

To maintain consistency with the previous experiment, we again use ground truth data generated from the Haar measure on $\text{SO}(4)$. With a fixed sample size of $n=2048$, we embedded the data into ambient spaces $\mathbb{R}^D$ of dimensions $D$ from $16$ to $512$ in powers of $2$ using random orthogonal matrices.

For each dimension $D$, we measured $W_1(\hat{p}_{\IDM}, p_0)$ and $W_1(\hat{p}_{0}, p_0)$. The IDM bandwidth was set as before, with the same $C_0 = \CZeroValueSecond$.

The right panel of Figure~\ref{fig:exp:combined} shows that the $W_1$ errors for both IDM and the memorized DM are largely unaffected by the ambient dimension $D$. Importantly, this demonstrates the  truncation of noise in the normal bundle via the inertia projection, a key aspect of our method predicted by Theorem~\ref{thm:main2}.
Combined with the findings from the previous section, these experiments confirm that the convergence of $W_1(\hat{p}_{\IDM}, p_0)$ is independent of the ambient dimension $D$ in both its rate and constant factor, thus successfully bypassing the curse of dimensionality.

\section{Proofs of the theoretical results}
\label{sec:proof}
We provide the proofs for the results in this section.

\subsection{NWE truncates normal perturbation}
Our first step is showing that NWE will effectively truncate perturbation in the normal tangent bundle. 
\begin{lemma}
\label{lem:ball}
Under Assumption \ref{aspt:manifold}, let $X_{(1)}, \dots, X_{(n)} \overset{\text{iid}}{\sim} p_0$. 
For all measurable $\displaystyle A \subseteq \mathcal{M}$, $\mu(A) = \int_A p_0(y) \, dV(y)$.
Fix any point \( x \in \mathcal{M} \) and radius \( r > 0 \), define \[I_{X,r} = \left\{ i \in [n] : \|X_{(i)} - x\| \le r \right\}\text{, and } p' := p_0(X \in B_{\mathcal{M}}(x, r)) \asymp  r^d.\] 

There is some universal constant \( c_0 > 0 \), so that if 
\( \log n \le c_0 n p' \), then for any fixed \( k > 0 \), there exists \( N_0 = N_0(k) \) such that for all \( n \ge N_0 \),
\[
\mathbb{P}\left( \left| \lvert I_{X,r} \rvert - n p' \right| \le \sqrt{2 n p' (1 - p') k \log n} + \frac{2}{3} k \log n \right) \ge 1 - \frac{2}{n^k}.
\]
In particular, if \( n p' \gg \log n \), then
\[
\left| \lvert I_{X,r} \rvert - n p' \right| \leq \sqrt{2 n p' k \log n}, \quad \text{with probability at least } 1 - \frac{2}{n^k}.
\]
\end{lemma}

\begin{proof}
Note that
\[
\lvert I_{X,r} \rvert = \sum_{i=1}^N \mathbf{1}_{X_{(i)} \in B_{\mathcal{M}}(x, r)},
\quad \text{where } \mathbf{1}_{X_{(i)} \in B_{\mathcal{M}}(x, r)} \overset{\text{iid}}{\sim} \mathrm{Bern}(p').
\]
Applying Bernstein’s inequality for Bernoulli sums, we obtain
\[
\mathbb{P}\left( \left| \lvert I_{X,r} \rvert - n p' \right| \le \sqrt{2 n p' (1 - p') k \log n} + \frac{2}{3} k \log n \right) \ge 1 - \frac{2}{n^k}.
\]

Finally, when \( n p' \gg \log n \), the first term dominates:
\[
\sqrt{2 n p' k \log n} \geq \frac{2}{3} k \log n,
\]
so the bound simplifies to
\[
\left| \lvert I_{X,r} \rvert - n p' \right| \leq  2\sqrt{2 n p' k \log n}.
\]
\end{proof}

\begin{cor}
Under condition of Lemma \ref{lem:ball}, if $np'\gg 1$ and $n$ is sufficiently large, 
with probability $\ge 1 - \frac{2}{n^k}$, $\lvert I_{X,r} \rvert \asymp n p'$.
\end{cor} 

\begin{cor}
Under condition of Lemma \ref{lem:ball}, there is a $\sigma_0>0$, if $\sigma<\sigma_0$ and  $n\sigma^{d}\gg 1$, then with probability $1-\sigma^k$, $\inf_{y\in\calM}|I_{y,\sigma}|\geq  1.$
\end{cor}
\begin{proof}
We pick a $\tfrac12\sigma$-net of $\calM$, which consists of point $\{y_j, j\in [M_\sigma]\}$. The cardinality $|M_\sigma|\leq 2^d\sigma^{-d}|\calM|$ by Proposition 3 of \cite{azangulov2024convergence}, where $|\calM|$ is the volume of the manifold. Then we apply Lemma \ref{lem:ball}, $|I_{y_j,\tfrac12\sigma}|>1$ with probability $1-n^{-(k+1)}$. So 
$\min_j\{ |I_{y_j,\tfrac12\sigma}|>1\}$ has probability $1-n^{-k}$. Finally we note that for any $y$, take $y_j$ being the closest point in the net to it, $I_{y_j,\frac12\sigma}\subset I_{y,\sigma}$.
\end{proof}

\begin{thm}
\label{thm:NWproj}
Given any fixed realization $X=[X_{(1)},\ldots,X_{(n)}]$ and $x\in \calM$ and $r>0$. Suppose 
\[
\inf_{y\in\calM}|I_{y,\sigma}|\geq  1, 
\]
and $n,\sigma^{-1}$ are sufficiently large.
Let $\xi\sim \mathcal{N}(0,I)$ with decomposition $\xi=\xi_{\mathcal{T}}+\xi_\bot$ on $T_x\calM\oplus N_x\calM$, then 
\[
\mathbb{P}\left( \left\| F_{X,\sigma}(x + \sigma \xi) - F_{X,\sigma}(\exp_x(\sigma \xi_{\mathcal{T}}))\right\| \lesssim (d(k+1)\log n\, \sigma)^2|X \right) \ge 1 - n^{-k}. 
\]
In particular, 
\[
\mathbb{P}\left( \left\| F_{X,\sigma}(x + \sigma \xi_\bot) - F_{X,\sigma}(x)\right\| \lesssim (d(k+1)\log n\, \sigma)^2|X \right) \ge 1 - n^{-k}. 
\]
\end{thm}

\begin{proof}
\textbf{Step 1: A Pythagorean approximation.}
We start by showing for $r=\sqrt{2d(k+1)\log n}\sigma,$ 
\begin{align}
\notag
\mathbb{P}\bigg(  \big|\|x + \sigma \xi &- X_{(i)}\|^2 - \|x + \sigma \xi_{\mathcal{T}} - X_{(i)}\|^2\\ &- \|\sigma \xi_\bot\|^2 \big| \lesssim \sqrt{k \log n} \, \sigma r^2,\ 
\forall i \in I_{x,r}|X\bigg) \ge 1 - n^{-k}.
\label{tmp:orthogonal}
\end{align}
We decompose
\begin{align}
\notag
\Delta_{1,i}:&=\|x + \sigma \xi - X_{(i)}\|^2 - \|x + \sigma \xi_{\mathcal{T}} - X_{(i)}\|^2 - \|\sigma \xi_\bot\|^2 \\
\label{tmp:decomp}
&= 2 \sigma \langle x + \sigma \xi_{\mathcal{T}} - X_{(i)}, \xi_\bot \rangle  = 2 \sigma \langle x - X_{(i)}, \xi_\bot \rangle.     
\end{align}
Since $i \in I_{X,r}$, $\exists v_i \in T_x \mathcal{M}$ with $\|v_i\|\leq 2r$ so that,  
\[
X_{(i)} = \exp_x(v_i) = x + v_i + O(\|v_i\|^2).
\]
This indicates, 
\[
\langle x - X_{(i)}, \xi_\bot \rangle = \langle x + v_i - \exp_x(v_i), \xi_\bot \rangle.
\]
Note that
\[
\|x + v_i - \exp_x(v_i)\| = \mathcal{O}(\|v_i\|^2) = \mathcal{O}(r^2).
\]
Also, $\langle x + v_i - \exp_x(v_i), \xi_\bot \rangle \sim \mathcal{N}(0, \|P_\bot(x + v_i - \exp_x(v_i))\|^2)$. Here, $P_\bot: \mathbb{R}^D \to N_x \mathcal{M}$ is the orthogonal projection onto $N_x \mathcal{M}$.\\

Let $t_i := \|P_\bot(x + v_i - \exp_x(v_i))\| = \mathcal{O}(r^2)$, $q_i := \langle x + v_i - \exp_x(v_i), \xi_\bot \rangle$, standard Gaussian tail inequality says
\[
\mathbb{P}(|q_i| > z|X) \le 2 \exp\left( -\frac{z^2}{2 t_i^2} \right).
\]
Apply union bound for the event $\mathcal{A}=\{\max_{i \in I_{X,r}} |q_i| > z\}$, then for some constant $C$
\[
\mathbb{P}\left( \mathcal{A} \,|\,X\right) \leq 2 |I_{X,r}|\exp\left( -\frac{z^2}{C r^2} \right).
\]
To have 
\[
2 n\exp\left( -\frac{z^2}{2 C r^4} \right) \leq n^{-k},
\]
It suffices to have $z \gtrsim r^2 \sqrt{(k+1) \log n}.$ Using \eqref{tmp:decomp}, we have \eqref{tmp:orthogonal}.

\noindent\textbf{Step 2: Show $\|\exp_x(\sigma \xi_{\mathcal{T}})-X_{(i)}\|^2-\|x+\sigma \xi_{\mathcal{T}}-X_{(i)}\|^2=\widetilde{O}_p(\sigma^3)$ for $i\in I_{X,r}$.} 

Denote the event $\mathcal{B}=\{\|\xi_{\mathcal{T}}\|\geq \sqrt{d(k+1)\log n }\}$. By Gaussian tail probability, we know $\Prob(\mathcal{B}|X)\leq n^{-k}$. Note under $\mathcal{B}^c$, for $i\in I_{x,r}$
\begin{align*}
|\Delta_{2,i}|&:=|\|\exp_x(\sigma \xi_{\mathcal{T}})-X_{(i)}\|^2-\|x+\sigma \xi_{\mathcal{T}}-X_{(i)}\|^2|\\
&\leq (\|\exp_x(\sigma \xi_{\mathcal{T}})-X_{(i)}\|+\|x+\sigma \xi_{\mathcal{T}}-X_{(i)}\|)\|\exp_x(\sigma \xi_{\mathcal{T}})-x-\sigma\xi_{\mathcal{T}}\|\\
&\leq 4Cr\|\sigma \xi_{\mathcal{T}}\|^2\leq 4C \sigma^3 (d(k+1)\log n )^{3/2}. 
\end{align*}

\noindent\textbf{Step 3: Showing $F_{X,\sigma}(x+\sigma\xi)=F_{X,\sigma}(\exp_x(\sigma \xi_{\mathcal{T}}))+\widetilde{O}_p(\sigma^2)$.}
Denote $x'=\exp_x(\sigma \xi_{\mathcal{T}})$, we note that if we let $z_i=\Delta_{1,i}+\Delta_{2,i}$
\begin{align*}
k_\sigma(x+\sigma \xi, X_{(i)})&=\exp(-\frac1{2\sigma^2}\|x+\sigma\xi-X_{(i)}\|^2)\\
&=\exp(-\frac1{2\sigma^2}\|x'-X_{(i)}\|^2-\frac1{2}\|\xi_\bot\|^2)\exp(-\frac{z_i}{2\sigma^2})\\
&=\exp(-\frac1{2}\|\xi_\bot\|^2)\exp(-\frac{z_i}{2\sigma^2})k_\sigma(x', X_{(i)}).
\end{align*}
Therefore we have
\begin{align*}
F_{X, \sigma}(x')= \frac{\sum\limits_{i=1}^n k_\sigma(x', X_{(i)}) X_{(i)}}{\sum\limits_{i=1}^n k_\sigma(x', X_{(i)})}
=\frac{\sum_{i\in [n]} w_i X_{(i)}}{\sum_{i\in [n]} w_i},
\end{align*}
and 
\begin{align*}
F_{X, \sigma}(x+\sigma \xi)=
\frac{\sum\limits_{i=1}^n k_\sigma(x', X_{(i)})\exp(-\frac{z_i}{2\sigma^2}) X_{(i)}}{\sum\limits_{i=1}^n k_\sigma(x', X_{(i)})\exp(-\frac{z_i}{2\sigma^2})}
=\frac{\sum_{i\in [n]} w_iu_i X_{(i)}}{\sum_{i\in [n]} w_iu_i},
\end{align*}
where $w_i:=k_\sigma(x', X_{(i)})$ and $u_i:=\exp(-\frac{z_i}{2\sigma^2})$. 
So 
\begin{align*}
&\left\|\frac{\sum_{i\in [n]} w_iu_i X_{(i)}}{\sum_{i\in [n]} w_iu_i}-
\frac{\sum_{i\in [n]} w_i X_{(i)}}{\sum_{i\in [n]} w_i}\right\|\\
&=\left\|\frac{\sum_{i\in [n]} w_iu_i (X_{(i)}-x')}{\sum_{i\in [n]} w_iu_i}-
\frac{\sum_{i\in [n]} w_i (X_{(i)}-x')}{\sum_{i\in [n]} w_i}\right\|\\
&\leq\left\|\frac{\sum_{i\in [n]} w_i (X_{(i)}-x')}{\sum_{i\in [n]} w_iu_i}-\frac{\sum_{i\in [n]} w_i (X_{(i)}-x')}{\sum_{i\in [n]} w_i}\right\|+\left\|\frac{\sum_{i\in [n]} w_i(u_i-1) (X_{(i)}-x')}{\sum_{i\in [n]} w_iu_i}\right\|,
\end{align*}
 Note that if $i\notin I_{x',r}$ with $r=\sqrt{2d(k+1)\log n}\sigma$  
\[
w_i=\exp(-\frac{1}{2\sigma^2} \|x'-X_{(i)}\|^2)
\leq \exp(-(k+1) \log n)\leq n^{-k-1}. 
\]
So 
\[
\sum_{i\notin I_{x',r}} w_i\leq n^{-k},\quad
\sum_{i\notin I_{x',r}} w_iu_i\leq n^{-k},\quad
\left\|\sum_{i\notin I_{x',r}} w_i (X_{(i)}-x')\right\|\leq n^{-k}\text{diam}(\calM). 
\]

Note that for all $i\in I_{x,r}$, under $\mathcal{A}^c\cap \mathcal{B}^c$, 
\[
\max\{|u_i-1|,\frac{|u_i-1|}{u_i}\}\leq \frac{|z_i|}{2\sigma^2}\leq C\sigma (d(k+1)\log n)^{3/2}<1/2,
\]
where we use the exponential ratio bound: 
for all \( x \in \mathbb{R} \) with small enough norm, we have:
\[
\left| e^{-x} - 1 \right| \le |x|.
\]
Meanwhile by our assumption of $X$,
\[
\sum_{i}w_i\geq  \sum_{i\in I_{x',\sigma}}w_i\geq 
\sum_{i\in I_{x',\sigma}}\exp(-1/2)
\geq e^{-1/2}.
\]
Meanwhile under $\mathcal{A}^c\cap \mathcal{B}^c$, $I_{x',\sigma}\subset I_{x,r}$, $\sum_i u_i w_i\geq 
\sum_{i\in I_{x',\sigma}}\exp(-1/2) u_i\geq \frac12 e^{-1/2}$. 

We bound each term using 
\begin{align*}
&\left\|\frac{\sum_{i\in [n]} w_i (X_{(i)}-x')}{\sum_{i\in [n]} w_iu_i}-\frac{\sum_{i\in [n]} w_i(X_{(i)}-x')}{\sum_{i\in [n]} w_i}\right\|\\
&=\frac{\|\sum_{i\in [n]} w_i (X_{(i)}-x')\|(\sum_{i\in [n]}w_i u_i|u_i^{-1}-1|)}{(\sum_{i\in [n]}w_i u_i)(\sum_{i\in [n]}w_i)}\\
&\leq \frac{\|\sum_{i\in [n]} w_i (X_{(i)}-x')\|(\sum_{i\in [n]}w_iu_i)\sigma (d(k+1)\log n)^{3/2}}{(\sum_{i\in [n]}w_i u_i)(\sum_{i\in [n]}w_i)}\\
&\leq \frac{(\|\sum_{i\in I_{x',r}} w_i (X_{(i)}-x')\|+\text{diam}(\calM)n^{-k})}{\sum_{i\in [n]}w_i}\sigma (d(k+1)\log n)^{3/2}\\
&\leq \left(r+C_d\text{diam}(\calM)n^{-k}\right)\sigma (d(k+1)\log n)^{3/2}\\
&\leq \sigma^2  (d(k+1)\log n)^2. 
\end{align*}
Similarly
\begin{align*}
&\left\|\frac{\sum_{i\in [n]} w_i(u_i-1) (X_{(i)}-x)}{\sum_{i\in [n]} w_iu_i}\right\|\\
&\leq \left\|\frac{\sum_{i\in I_{x,r}} w_i(u_i-1) (X_{(i)}-x)}{\sum_{i\in [n]} w_iu_i}\right\|+\left\|\frac{\sum_{i\in I^c_{x,r}} w_i(u_i-1) (X_{(i)}-x)}{\sum_{i\in [n]} w_iu_i}\right\|\\
&\leq r\left|\frac{\sum_{i\in I_{x,r}} w_i(u_i-1) }{\sum_{i\in [n]} w_iu_i}\right|+\frac{n^{-k}\text{diam}(\calM)}{\sum_{i\in [n]} w_iu_i}\leq 2e\sigma^2  ((k+1)\log n)^2.
\end{align*}
Finally, note that $\xi_{\mathcal{T}}$ and $\xi_\bot$ are independent. So all the discussion above holds if we add in the condition that $\xi_{\mathcal{T}}=0$. 
\end{proof}

\subsection{Concentration for KE on $\calM$}
\label{sub:KDEproof}
We will refer to functions of form 
\[
\widehat{f}=\frac{1}{n}\sum_{i=1}^n k_{\sigma}(x,X(i)) f(X(i)) 
\]
a kernelised estimator (KE) for function $f$. Notably, the NWE \eqref{eqn:NWE} is the ratio of KE for $f(x)=x$ and $f(x)\equiv 1$, and the manifold KDE is the KE for $f(x)\equiv 1$. 

Our second step is studying  that KE on the manifold.  We note that \cite{ozakin2009submanifold} has already proved the convergence of manifold KDE in $L_2$ norm. 
\begin{lemma}
 \label{lem: pointwise_bias} 
  Under Assumption \ref{aspt:manifold}, suppose $X_{(1)},\ldots, X_{(n)}$ are i.i.d. samples from $p_0$. Denote \( K(r) = (2\pi)^{-d/2} e^{-r^2/2} \) so that \(k_\sigma(x, X_{(i)}) = \frac{1}{\sigma^d} K\left( \frac{\|x-X_{(i)}\|}{\sigma} \right)\). 
  Given a $C^2$ function $f$,
  let the kernelized regressor at a point \( x \in \mathcal{M} \) with bandwidth \( \sigma > 0 \) can be written as:
\[
\widehat{f}(x) = \frac{1}{n} \sum\limits_{i=1}^n k_\sigma(x, X_{(i)}) f(X_{(i)}):= \frac{1}{n} \sum_{i=1}^n \frac{1}{\sigma^d} K\left( \frac{\|x - X_{(i)}\|}{\sigma} \right)f(X_{(i)}),
\]
where \( \|x-y\| \) is the distance in the \(\mathbb{R}^D\).
Then the pointwise bias
\[
b(x) := \mathbb{E}[\hat{f}(x)] - p_0(x)f(x)
\]
satisfies $b(x)=\frac{1}{2} \sigma^2 \Delta (f(x)p_0(x))+o(\sigma^2)=O(\sigma^2)$. 
\end{lemma}
\begin{proof}
Denote \( q(x)=p_0(x) f(x)\). 
    We analyze the pointwise bias of the kernel density estimator:
\[
\text{Bias}(x) := \mathbb{E}[\hat{f}(x)] - q(x)
= \int_{\mathcal{M}} \frac{1}{\sigma^d} K\left( \frac{\|x - y\|}{\sigma} \right) [q(y) - q(x)] \, dV(y).
\]
Split the integral into a local and a tail region:
\[
\begin{aligned}
\text{Bias}(x) 
= &\underbrace{\int_{B_{\sigma \gamma}(x)} \frac{1}{\sigma^d} K\left( \frac{\|x - y\|}{\sigma} \right) [q(y) - q(x)] dV(y)}_{\text{Bias}_{\text{local}}(x)} \\
&+ \underbrace{\int_{\mathcal{M} \setminus B_{\sigma \gamma}(x)} \frac{1}{\sigma^d} K\left( \frac{\|x - y\|}{\sigma} \right) [q(y) - q(x)] dV(y)}_{\text{Bias}_{\text{tail}}(x)},
\end{aligned}
\]
where \(\gamma\) will be set later.\\
\textbf{Step 1: tail Term.}
We estimate the tail bias term directly using the exponential decay of the Gaussian kernel and the compactness of \(\mathcal{M}\).
From the definition:
\[
\text{Bias}_{\text{tail}}(x) := \int_{\mathcal{M} \setminus B_{\sigma \gamma}(x)} \frac{1}{\sigma^d} K\left( \frac{\|x - y\|}{\sigma} \right) [q(y) - q(x)] \, dV(y).
\]
Since \(q\) is \(C^2\) on compact \(\mathcal{M}\), it is uniformly bounded, so there exists \(C_1 > 0\) such that:
\[
|q(y) - q(x)| \le C_1 \quad \text{for all } y \in \mathcal{M}.
\]
For \(y \in \mathcal{M} \setminus B_{\sigma \gamma}(x)\), we have \(\|x - y\| \ge \sigma \gamma\), so the Gaussian kernel satisfies:
\[
K\left( \frac{\|x - y\|}{\sigma} \right) \le (2\pi)^{-d/2} \exp\left( -\frac{\gamma^2}{2} \right).
\]
Thus, the entire integrand is uniformly bounded on the integration domain:
\[
\left| \frac{1}{\sigma^d} K\left( \frac{\|x - y\|}{\sigma} \right) [q(y) - q(x)] \right| \le \frac{C_2}{\sigma^d} \exp\left( -\frac{\gamma^2}{2} \right).
\]
Hence,
\[
|\text{Bias}_{\text{tail}}(x)| \le \frac{C_2}{\sigma^d} \exp\left( -\frac{\gamma^2}{2} \right) \cdot \mathrm{Vol}(\mathcal{M}).
\]
Now set \( \gamma = R \sqrt{\log(1/\sigma)} \), and then:
\[
\exp\left( -\frac{\gamma^2}{2} \right) = \sigma^{R^2/2},
\]
so we have:
\[
|\text{Bias}_{\text{tail}}(x)| \le C_3 \cdot \sigma^{-d} \cdot \sigma^{R^2/2} = C_3 \cdot \sigma^{R^2/2 - d}.
\]
Choosing \(R^2 > 2d + 4\), we obtain:
\[
|\text{Bias}_{\text{tail}}(x)| =  O(\sigma^3),
\]
which shows that the tail contribution is negligible compared to the leading order bias.\\
\textbf{Step 2: local Term.}
In geodesic normal coordinates centered at \(x\), set 
\[
y = \exp_x(\sigma u),\quad u \in T_x \mathcal{M}, \quad \|u\| \le \gamma,
\]
then:
\[
r(u):=\frac{1}{\sigma}(\|x-y\| -\sigma \|u\|)\leq C\sigma^2\gamma^3, \quad  dV(x) = \sigma^d \sqrt{\det g_{ij}(\sigma u)} \, du.
\]
Use the expansion with $x_u$ being a point along $\exp_x(\sigma s), s\in [0,u]$ \cite{Boumal_2023},
\[
q(\exp_x(\sigma u)) = q(x) + \sigma \nabla q(x) \cdot u + \frac{\sigma^2}{2} u^\top \nabla^2 q(x_u) u + O(\sigma^3).
\]
Here \( \nabla q(x) \) is the \emph{Riemannian gradient} of \( q \) at \( x \), i.e., the unique vector in \( T_x\mathcal{M} \) such that for all \( v \in T_x\mathcal{M} \),\(\langle \nabla q(x), v \rangle = D_v q(x),\) where \( D_v q(x) \) is the directional derivative of \( q \) at \( x \) along \( v \) and \( \nabla^2 q(x_u) \) is the \emph{Riemannian Hessian} of \( q \) at the point \( x_u \), defined via the Levi-Civita connection that
    \(
    \nabla^2 q(x_u)(v, w) = \langle \nabla_v \nabla q, w \rangle \big|_{x_u},\text{ for } v, w \in T_{x_u}\mathcal{M}.
    \)\\
From the expansion of the metric tensor in normal coordinates (see page 91 of \cite{Chavel_2006}):
\[
g_{ij}(\sigma u) = \delta_{ij} - \frac{\sigma^2}{3} \sum_{k,l} R_{ikjl}(x) u^k u^l,
\]
where \( R_{ikjl}(x) \) is the Riemann curvature tensor at point \( x \) in normal coordinates.
Applying the determinant identity \( \sqrt{\det(I + A)} = 1 + \frac{1}{2} \text{tr}(A) + O(\|A\|^2) \), we find:
\[
\operatorname{tr}\left( -\frac{\sigma^2}{3} R_{ikjl} u^k u^l \right) = -\frac{\sigma^2}{3} \operatorname{Ric}_{ij}(x) u^i u^j = -\frac{\sigma^2}{3} \operatorname{Ric}_{x}(u, u).
\]
Therefore:
\[
\sqrt{\det g_{ij}(\sigma u)} = 1 - \frac{\sigma^2}{6} \operatorname{Ric}_x(u, u) + O(\sigma^3) = 1 + O(\sigma^2).
\]
Also, since:
\[
|K(\|u\|+r(u))-K(\|u\|)|\leq C |r(u)|\leq C\sigma^2\gamma^3,
\]
we obtain:
\[
\begin{aligned}
\text{Bias}_{\text{local}}(x)
&= \int_{\|u\| \le \gamma} K(\|u\|) \left[ \sigma \nabla q(x) \cdot u + \frac{\sigma^2}{2} u^\top \nabla^2 q(x_u) u \right] du. + O(\sigma^3)
\end{aligned}
\]
By symmetry of the Gaussian kernel:
\[
\int_{\|u\| \le \gamma} K(\|u\|) u \, du = 0, \quad \int_{\|u\| \le \gamma} K(\|u\|) u_i u_j \, du = \mu_2(\gamma) \delta_{ij},
\]
where
\[
\mu_2(\gamma) := \int_{\|u\| \le \gamma} K(\|u\|) u_1^2 \, du.
\]
Since:
\[
\mu_2(\gamma) \le \int_{\mathbb{R}^d} K(\|u\|) u_1^2 \, du = \mu_2,
\]
since $\|\nabla^2 q(x_u)-\nabla^2 q(x)\|=o(1)$, we have:
\[
\text{Bias}_{\text{local}}(x) = \frac{\sigma^2}{2} \mu_2 \Delta q(x) + o(\sigma^2).
\]
Combining both terms, we obtain:
\[
\text{Bias}(x) := \mathbb{E}[\hat{f}(x)] - q(x)
= \frac{\sigma^2}{2} \Delta q(x) + o(\sigma^2 ).
\]
\end{proof}

\begin{lemma}
\label{lem:pointwise_variance}
   With the same assumption as lemma \ref{lem: pointwise_bias}, then there exists a constant \( C_v > 0 \) such that the pointwise variance of \(\hat{f}(x)\) satisfy:
\[
 \mathrm{Var}[\hat{f}_{KDE}(x)]  \leq \frac{C_v}{n \sigma^d}.
\]
\end{lemma}

\begin{proof}
Since
\[
\hat{f}(x) = \frac{1}{n} \sum_{i=1}^n k_\sigma(x, X_{(i)})f(X_{(i)}),
  \quad \text{where } 
k_\sigma(x, X_{(i)}) = \frac{1}{\sigma^d} K\left( \frac{\|x-X_{(i)}\|}{\sigma} \right).
\]
Then
\[
\mathrm{Var}[\hat{f}(x)] = \frac{1}{n} \mathrm{Var}[k_\sigma(x, X_{(1)})f(X_{(1)})] \le \frac{1}{n} \mathbb{E}[k_\sigma^2(x, X_{(1)}){f^2(X_{(1)})}].
\]
We compute
\[
\mathbb{E}[k_\sigma^2(x, X_{(1)})f^2(X_{(1)})] = \int_{\mathcal{M}} \left( \frac{1}{\sigma^d} K\left( \frac{\|x-y\|}{\sigma} \right) \right)^2 f^2(y)p_0(y) \, dV(y).
\]
Split the integral into two parts:
\[
\int_{\mathcal{M}} = \int_{B_{\mathcal{M}}(x, R\sigma)} + \int_{\mathcal{M} \setminus B_{\mathcal{M}}(x, R\sigma)} =: I_{\text{local}} + I_{\text{tail}}.
\]
For the tail region, Gaussian tail decay gives
\[
K\left( \frac{\|x-y\|}{\sigma} \right)^2 \le \exp\left( -\frac{\|x-y\|^2}{\sigma^2} \right) \le e^{-R^2},
\]
and that \(|f^2p_0(x)| \le C_1\) on \(\mathcal{M}\), so
\[
I_{\text{tail}} \le \frac{C_1}{\sigma^{2d}} e^{-R^2} \cdot \operatorname{vol}(\mathcal{M}) = O(\sigma^{-2d} e^{-R^2}).
\]
Choosing \( R^2 = (d + c) \log(1/\sigma) \Rightarrow e^{-R^2} = \sigma^{d + c} \), we have $I_{\text{tail}} = O(\sigma^{-d - c}).$
Now for the local region, use the exponential map \( x = \exp_y(\sigma z) \), with \( z \in B(0,R) \subset T_y \mathcal{M} \). Then:
\[
r(z):=\frac{1}{\sigma}(\|x-y\| -\sigma \|z\|)\leq C\sigma^2 R^3, \quad  dV(x) = \sigma^d \sqrt{\det g_{ij}(\sigma z)} \, dz.
\]
and use the expansion,
$|\sqrt{\det g_{ij}(\sigma z)} - 1|\leq C\sigma^2,$
as well as:
\[
|K(\|z\|+r(z))-K(\|z\|)|\leq C |r(z)|\leq C\sigma^2R^3.
\]
Hence,
\[
\int_{B_{\mathcal{M}}(y,R\sigma)} K^2\left( \frac{\|x-y\|}{\sigma} \right) dV(x)
= \sigma^d \int_{B(0,R)} K^2(\|z\|) (1 + O(\sigma^2)) \, dz = \sigma^d \omega_d + O(\sigma^{d+2}),
\]
where \( \omega_d := \int_{\mathbb{R}^d} K^2(\|z\|) dz \). Therefore,
\[
\mathbb{E}[k_\sigma^2(x, X_{(1)})] \le \frac{C_1}{\sigma^{2d}} (\sigma^d \omega_d + O(\sigma^{d+2})) = \frac{C_1 \omega_d}{\sigma^d} + O(\sigma^{-d-2}).
\]
Therefore, we obtain
\[
\mathrm{Var}[\hat{f}(x)]  \leq \frac{C_v}{n \sigma^d},
\quad \text{with } C_v := C_1 \omega_d.
\]
\end{proof}
\begin{lemma}
\label{lem:denominator_F}
Under the same Assumption as lemma \ref{lem: pointwise_bias},
then, with high probability \(1-n^{-k}\), we have
\[
\widehat{f}(x) = p(x)f(x) + O(\sqrt{k\log n}\sigma^2).
\]
\end{lemma}

\begin{proof}[Proof of Proposition \ref{prop:KDEp0}]
Denote \(q(x)=p_0(x)f(x)\). By Lemma \ref{lem: pointwise_bias}, we know that
\[
\text{Bias}(x) := \mathbb{E}[\hat{f}(x)] - q(x)
=  O(\sigma^2).
\]
and by Lemma \ref{lem:pointwise_variance}, we know $\mathrm{Var}(\hat{f}(x)) \leq \frac{C}{n \sigma^d}.$  Define the centered random variable:
\[
Y_i := \frac{1}{\sigma^d} K\left( \frac{\|x - X_{(i)}\|}{\sigma} \right)f(X_{(i)}) - q(x),
\]
so that \( \hat{f}(x) - q(x) = \frac{1}{n} \sum_{i=1}^n Y_i \). We analyze its concentration using Bernstein’s inequality.
First, note the following properties:
\begin{itemize}
    \item The kernel function is bounded, so:
\[
|Y_i| \le \frac{1}{\sigma^d} + |q(x)| \le \frac{C}{\sigma^d},
\]
for some constant \(C > 0\);
\item The variance of \(Y_i\) is:
\[
\mathrm{Var}(Y_i) = \mathrm{Var}\left( \frac{1}{\sigma^d} K\left( \frac{\|x - X_{(i)}\|}{\sigma} \right) f(X_{(i)})\right) \le \frac{C'}{\sigma^d};
\]
\item  The mean satisfies:
\[
\mathbb{E}[Y_i] = \mathbb{E}[\hat{f}(x)] - q(x) = O(\sigma^2).
\]
\end{itemize}
Now apply Bernstein’s inequality: for i.i.d. \(Y_i\) with mean \(\mu = \mathbb{E}[Y_i]\), variance proxy \(\sigma_Y^2\), and bounded by \(B\), we have:
\[
\mathbb{P}\left( \left| \frac{1}{n} \sum_{i=1}^n Y_i - \mu \right| > t \right)
\le 2 \exp\left( - \frac{n t^2}{2 \sigma_Y^2 + \frac{2}{3} B t} \right).
\]
Set \( t = \sigma^2 \sqrt{C k \log n} \), then:
\[
\mathbb{P}\left( \left| \hat{f}(x) - \mathbb{E}[\hat{f}(x)] \right| > \sigma^2 \sqrt{C k\log n} \right)
\le 2 \exp\left( - \frac{Ck n \log n \sigma^4}{2 C'/\sigma^d + \frac{2}{3} C \sigma^2 / \sigma^d} \right).
\]
The denominator is \( O(1/\sigma^d) \), so by setting sufficiently large $C$,
\[
\mathbb{P}\left( \left| \hat{f}(x) - \mathbb{E}[\hat{f}(x)] \right| > \sigma^2 \sqrt{C k\log n}\right)
\le  \exp\left( - k n \log n \sigma^{d+4} \right).
\]
To ensure this deviation probability is at most \(n^{-k}\) for any \(k > 0\), we require:
\[
\exp\left( - k n \log n \sigma^{d+4} \right) \le n^{-k}
\quad \Longleftrightarrow \quad
n \ge \frac{1 }{\sigma^{d+4}}.
\]
\end{proof}

\subsection{The KE are Lipschitz}
\label{sec:KELip}
\begin{lemma}
\label{lem:KDELip}
For $\widehat{p}_{\text{KDE}}(x)$ is $\sigma^{-(d+1)}$-Lipschitz continuous almost surely for all $x$.
\end{lemma}
\begin{proof}
Recall \(
\widehat{p}_{\text{KDE}}(x) := \frac{1}{n} \sum_{i=1}^n \frac{1}{\sigma^d} K\left(-\frac{\|x-X_{(i)}\|^2}{2\sigma^2} \right).
\)
A direct calculation gives 
\[
\nabla \widehat{p}_{\text{KDE}}(x) =\frac{1}{n} \sum_{i\in [n]} \frac{1}{\sigma^{d+2}} \exp\left(-\frac{\|x-X_{(i)}\|^2}{2\sigma^2}\right)(x-X_{(i)}).
\]
Note that by calculus,
\[
\exp\left(-\frac{\|x-X_{(i)}\|^2}{2\sigma^2}\right)\|x-X_{(i)}\|
\leq \sigma. 
\]
To see this, let \( f(r) := r \exp\left( -\frac{r^2}{2\sigma^2} \right) \) for \( r \ge 0 \). By taking its derivative:
\[
f'(r) = \exp\left( -\frac{r^2}{2\sigma^2} \right) \left( 1 - \frac{r^2}{\sigma^2} \right).
\]
Setting \( f'(r) = 0 \), we find the critical point at \( r = \sigma \). Since \( f'(r) > 0 \) for \( r < \sigma \) and \( f'(r) < 0 \) for \( r > \sigma \), this is the unique global maximum.
Evaluating at \( r = \sigma \), we have:
\[
f(\sigma) = \sigma \exp\left( -\frac{1}{2} \right).
\]
Since \( \exp(-1/2) < 1 \), it follows that \( f(\sigma) < \sigma \), and hence:
\[
f(r) \le \sigma \quad \text{for all } r \ge 0.
\]
We have our claim.
\end{proof}

\begin{lemma}
\label{lem:NWLip}
    $F_{X, \sigma}(z)$ is Lipschitz continuous with
\[
\|F_{X, \sigma}\|_{\mathrm{Lip}} := \sup_{x,y \in \mathbb{R}^D, x \ne y} \frac{\|F_{X, \sigma}(x) - F_{X, \sigma}(y)\|}{\|x - y\|} \le \frac{\mathrm{diam}(\mathcal{M})^2}{4 \sigma^2}.
\]

\end{lemma}
\begin{proof}
    A direct calculation gives
\[
\nabla F_{X, \sigma}(z) = \frac{1}{\sigma^2} \sum_{i\in[n]} w_i(z) (X_{(i)} - F_{X,\sigma}(z))(X_{(i)} - F_{X,\sigma}(z))^T,
\]
where
\[
w_i(z) = \frac{k_\sigma(z, X_{(i)})}{\sum_{j\in[n]} k_\sigma(z, X_{(j)})}, \quad 1 \le i \le n.
\]

Fix $z \in \mathbb{R}^D$.  
Consider a new discrete distribution $Y_z$ taking values in $\{X_{(1)}, \dots, X_{(n)}\}$ with $\mathbb{P}(Y_z = X_{(i)}) = w_i(z)$.  
Thus $\mathbb{E}[Y_z] = F_{X,\sigma}(z)$.\\

We have
\begin{align*}
    \|\nabla F_{X, \sigma}&(z)\|_{\mathrm{op}} 
        = \sup_{\|v\|=1} \|v^T \nabla F_{X, \sigma}(z) v\|\\[1mm]
        &\le  \sum_{i=1}^N \frac{w_i(z)}{\sigma^2} \|X_{(i)} - F(z)\|^2= \frac{\mathrm{Var}(Y_z)}{\sigma^2} = \frac{\mathbb{E}\Bigl[\|Y_z - Y_z'\|^2\Bigr]}{\sigma^2} \le \frac{\mathrm{diam}(\mathcal{M})^2}{4 \sigma^2},
    \end{align*}
where $Y_z'$ is an independent copy of $Y_z$.\\

Hence
\[
\|F_{X, \sigma}\|_{\mathrm{Lip}} \le \sup_{z \in \mathbb{R}^D} \|\nabla F_{X, \sigma}(z)\|_{\mathrm{op}} \le \frac{\mathrm{diam}(\mathcal{M})^2}{4 \sigma^2}. 
\]
\end{proof}

\subsection{$L_\infty$ error of the KEs}
\label{sec:linftyproof}
\begin{pro}
\label{prop:NWonM}
Recall the NWE
\[
F_{X, \sigma}(z) := \frac{\sum\limits_{i=1}^N k_\sigma(z, X_{(i)}) X_{(i)}}{\sum\limits_{i=1}^N k_\sigma(z, X_{(i)})}.
\]
Suppose $\sigma=n^{-\frac{1}{d+4}}$ and
 Assumption \ref{aspt:manifold}  holds.
Then for any given $x\in \calM$, with probability $1-n^{-k}$, 
   \[
\sup_{x\in\calM}\|F_{X,\sigma}(x)- x\| =O(\sigma^2\sqrt{(2d+k)\log n}).
\]
\end{pro}
\begin{proof}
First, let us fix a point $x\in\calM$. From previous lemma \ref{lem:denominator_F} with $f(x)=x$ and $f(x)\equiv 1$, we have
\[
\sum_{i=1}^n \frac{1}{n\sigma^d} K\left( \frac{\|x-X_{(i)}\|}{\sigma} \right) X_{(i)}
= p_0(x) x + O(\sigma^2\sqrt{k\log n}),
\]
\[
\sum_{i=1}^n \frac{1}{n\sigma^d} K\left( \frac{\|x-X_{(i)}\|}{\sigma} \right)
= p_0(x) + O(\sigma^2\sqrt{k\log n}),
\]
with probability at least \(1 - n^{-k}\).
Since $p_0(x)\geq 1/C$, we have 
\[
\|F_{X,\sigma}(x)-x\|\leq \sigma^2 \sqrt{k\log n}. 
\]
Next, we pick a $\sigma^{(4+d)}$-net $\{y_1,\ldots, y_{M_\sigma}\}$ of $\calM$. Then its cardinality $M_{\sigma}\leq C_d\sigma^{-(4+d)d}$ by Assumption \ref{aspt:manifold} and Proposition 3 of \cite{azangulov2024convergence}. Denote event 
\[
\mathcal{B}_j=\{|F_{X,\sigma}(y_j)-y_j|>\sigma^2\sqrt{(2d+k)\log n}\},\quad \mathcal{B}=\cup_{j=1}^{M_\epsilon} \mathcal{B}_j.
\]
By Proposition \ref{prop:NWonM},  $\Prob(\mathcal{B})\leq 
\sum_{j\leq M_\epsilon}\Prob(\mathcal{B}_j)\leq \sigma^{2(4+d)d} n^{-2d-k}\leq n^{-k}$. 
By Lemma \ref{lem:NWLip}, $F_{X,\sigma}$ is $\sigma^{-2}$-Lipschitz. So when $\mathcal{B}^c$ takes place, we have 
\begin{align*}
&\sup_{y\in \calM}\|F_{X,\sigma}(y)-y\|\\
&\leq \sup_{y\in \calM}\min_{j\in [M_\epsilon]}\{\|F_{X,\sigma}(y)-F_{X,\sigma}(y_j)\|+\|y-y_j\|+\|F_{X,\sigma}(y_j)-y_j\|\}\\
&\leq \sup_{y\in \calM}\min_{j\in [M_\epsilon]}\{ (1+\sigma^{-2})\|y-y_j\|+C\sqrt{(2d+k)\log n}\sigma^2\}\\
&\leq C\sqrt{(2d+k)\log n}\sigma^2. 
\end{align*}



\end{proof}

\begin{proof}[Proof of Theorem \ref{thm:main2}]
Theorem \ref{thm:main2} comes a corollary of Theorem \ref{thm:NWproj} and Proposition \ref{prop:NWonM}, simply note that $\exp_x(\sigma \xi_{\mathcal{T}})\in\calM$.  
\end{proof}

\begin{proof}[Proof of Proposition \ref{prop:KDEp0}]
We pick a $\sigma^{4+d}$-net $\{y_1,\ldots, y_{M_\sigma}\}$ of $\calM$. Then its cardinality $M_{\sigma}\leq C_d\sigma^{-4d}$ by Assumption \ref{aspt:manifold} and Proposition 3 of \cite{azangulov2024convergence}.
Consider the event
\[
\mathcal{C}_j=\{|\widehat{p}_{\text{KDE}}(y_j)-p(y_j)|>\sqrt{(2d+k)\log n}\sigma^2\},\quad \mathcal{C}=\cup_{j=1}^{M_\sigma} \mathcal{C}_j.
\]
Lemma \ref{lem:denominator_F} shows that with probability $$\Prob(\mathcal{C})\leq 
\sum_{j\leq M_\epsilon}\Prob(\mathcal{C}_j)\leq C_d\sigma^{-(d+4)d} n^{-(k+2d)}\leq C_dn^{-k}$$
Therefore, by Lemma \ref{lem:KDELip}
\begin{align*}
|\widehat{p}_{\text{KDE}}(y)-p(y)|&\leq \min_{j\in [M_\epsilon]}\{|\widehat{p}_{\text{KDE}}(y)-\widehat{p}_{\text{KDE}}(y_j)|+|p(y)-p(y_j)|+|\widehat{p}_{\text{KDE}}(y_j)-p(y_j)|\}\\
&\leq \min_{j\in [M_\epsilon]}\{ (L+\sigma^{-d-2})\|y-y_j\|+C\sigma^2\}\leq C\sigma^2. 
\end{align*}
\end{proof}

\subsection{Samples from NWE on $\calM$ is close to a KDE}
\label{sec:KDEproof}

\begin{lemma}[Expectation Representation of Manifold KDE Integral]
\label{lem:KDEexp}
For sufficiently small \( \sigma \), given any $X=[X_{(1)},\ldots, X_{(n)}]$ and Lip-1 $f$
\[
|\int_{\mathcal{M}} \hat{p}_{\mathrm{KDE}}(x) f(x) \, dV(x) - \mathbb{E}_{U, \xi_{\mathcal{T}}}\left[ f\left( \exp_{X_U}(\sigma \xi_{\mathcal{T}}) \right) \right|\leq (2(d+4)\log (2\pi/\sigma))^{\frac{3}{2}}\sigma^2.
\]
\end{lemma}

\begin{proof}
Pick any Lip-1 $f$. Without loss of generality, we assume $f(x_0)=0$ for some $x_0\in \calM$. Then we have $f(x)\leq \text{diam}(\calM)$. We start by 
\[
\int_{\mathcal{M}} \hat{p}_{\text{KDE}}(x) f(x) \, dV(x) = \frac{1}{n} \sum_{i=1}^n \int_{\mathcal{M}} \frac{1}{\sigma^d} K\left( \frac{\|x-X_{(i)}\|}{\sigma} \right) f(x) \, dV(x) =: \frac{1}{n} \sum_{i=1}^n I_i.
\]
We split the integral into a local and a tail region:
\[
I_i = \int_{\mathcal{M}} \frac{1}{\sigma^d} K\left( \frac{\|x-X_{(i)}\|}{\sigma} \right) f(x) \, dV(x) = I_{\text{local}} + I_{\text{tail}},
\]
where \( I_{\text{local}} \) integrates over \( B_{\sigma \gamma}(X_{(i)}) \subset \mathcal{M} \), a radius $\sigma \gamma$ neighborhood around \(X_{(i)}\) in $\calM$, and \( I_{\text{tail
}} \) is the integral over the complement. Here $\gamma$ is chosen as $\sqrt{2(d+4)|\log (2\pi/\sigma)|}$

Using the exponential decay of the Gaussian kernel and compactness of \(\mathcal{M}\), and the Lipschitz-ness of $f$,
the tail term satisfies:
\[
|I_{\text{tail}}| \le C \cdot (\sqrt{2\pi}\sigma)^{-d} \cdot \exp\left( - \frac{\gamma^2}{2} \right) = O(\sigma^3).
\] 
We consider the integral:
\[
I_{\text{local}} := \int_{B_{\sigma \gamma}(X_{(i)})\cap \calM} \frac{1}{\sigma^d} K\left( \frac{\|x-X_{(i)}\|^2}{\sigma} \right) f(x) \, dV(x),
\]
and restrict to a small neighborhood \( B_{\sigma \gamma}(X_{(i)}) \) around \(X_{(i)}\), where the exponential map is a diffeomorphism. We change variables to geodesic normal coordinates centered at \(X_{(i)}\), writing:
\[
x = \exp_{X_{(i)}}(\sigma u), \quad u \in T_{X_{(i)}} \mathcal{M} \cong \mathbb{R}^d.
\]
Under this map, we have:
\[
r(u):=\frac{1}{\sigma}(\|x-X_{(i)}\| -\sigma \|u\|)\leq C\sigma^2\gamma^3, \quad  dV(x) = \sigma^d \sqrt{\det g_{ij}(\sigma u)} \, du.
\]
Hence, the integral becomes:
\[
I_{\text{local}}=\int_{\|u\| \le \gamma} K(\|u\|+r(u)) f(\exp_{X_{(i)}}(\sigma u)) \sqrt{\det g_{ij}(\sigma u)} \, du.
\]
We use the expansion:
\[
|\sqrt{\det g_{ij}(\sigma u)} - 1|\leq  \frac{\sigma^2}{6} \operatorname{Ric}_{X_{(i)}}(u, u) + o(\sigma^2)\leq C\sigma^2,
\]
as well as
\[
|K(\|u\|+r(u))-K(\|u\|)|\leq C |r(u)|\leq C\sigma^2\gamma^3.
\]
So we can conclude that 
\begin{equation}
\label{temp:1side}
\left|I_i- \int_{\|u\| \le \gamma} K(\|u\|) f(\exp_{X_{(i)}}(\sigma u)) \, du\right|\leq C\sigma^2 \gamma^3.
\end{equation}
On the other hand,
\[
\mathbb{E}_{U, \xi_{\mathcal{T}}}[f(\exp_{X_U}(\sigma \xi_{\mathcal{T}}))] = \frac{1}{n} \sum_{i=1}^n \mathbb{E}_{\xi_{\mathcal{T}}}[f(\exp_{X_{(i)}}(\sigma \xi_{\mathcal{T}}))],
\]
since \( U \sim \mathrm{Unif}([n]) \) and \( \xi_{\mathcal{T}} \sim \mathcal{N}(0, I_d) \) is independent of $U$. Each expectation over \( \xi_{\mathcal{T}} \) is equivalent to
\[
\mathbb{E}_{\xi_{\mathcal{T}}}[f(\exp_{X_{(i)}}(\sigma \xi_{\mathcal{T}}))] = \int_{\mathbb{R}^d} f(\exp_{X_{(i)}}(\sigma u)) K(\|u\|) du. 
\]
Compare with \eqref{temp:1side}, we can reach our conclusion when we note by $\chi^2_d$ tail bound, 
\begin{align*}
\int_{u:\|u\|\geq \gamma} f(\exp_{X_{(i)}}(\sigma u)) K(\|u\|) du&\leq \text{diam}(\calM)\int_{u:\|u\|\geq \gamma}  K(\|u\|) du\\
&\leq \text{diam}(\calM) \gamma^2\exp(-\gamma^2/2)\leq C\sigma^2\gamma^3. 
\end{align*}
Thus, the KDE integral and the Gaussian expectation differ only by \(O(\sigma^2\gamma^3)\), and we may write:
\[
\mathbb{E}_{\xi_{\mathcal{T}}}[f(\exp_{X_{(i)}}(\sigma \xi_{\mathcal{T}}))] = I_i + O(\sigma^2\gamma^3).
\]

Taking expectation over \(U \sim \mathrm{Unif}([n])\), we obtain
\[
\mathbb{E}_{U, \xi_{\mathcal{T}}}[f(\exp_{X_U}(\sigma \xi_{\mathcal{T}}))] = \frac{1}{n} \sum_{i=1}^n \mathbb{E}_{\xi_{\mathcal{T}}}[f(\exp_{X_{(i)}}(\sigma \xi_{\mathcal{T}}))] = \frac{1}{n} \sum_{i=1}^n I_i + O(\sigma^2\gamma^3),
\]
establishing the desired correspondence and completing the proof.

\end{proof}

\subsection{Proof of the main results}
\label{sec:proofmain}


\begin{proof}[Proof of Theorem \ref{thm:main3}]
Let $\delta=h^2$. We will first assume $Z_0\sim \phat_T$. We start by noting the following holds for all test function  
\[
\int \hat{p}_{\text{IDM}}(x)f(x)dx=\E[f(\frac{1}{\alpha_\delta}F_{\alpha_\delta X,\sigma}(\alpha_\delta X_U+\sigma \xi))|X]
\]
Note
\[
\frac{1}{\alpha_\delta}F_{\alpha_\delta X,\sigma_\delta}(\alpha_\delta X_U+\sigma_\delta \xi)=
F_{X,\sigma'}(X_U+\sigma' \xi)
\]
where
\[
\sigma'=\frac{\sigma_\delta}{\alpha_\delta}=
\frac{\sqrt{1-e^{-2\delta}}}{e^{-\delta}}=h+O(h^3).
\]
In particular, almost surely,
\[
\|F_{X,\sigma'}(X_U+\sigma' \xi)\|=
\|\sum_{i\in[n]} w_i  X_{(i)}\|\leq \text{diam}(\calM).
\]
So by $f$ being a Lip-1 function, we know 
\[
|f(F_{X,\sigma'}(X_U+\sigma' \xi))-f(0)|\leq \text{diam}(\calM).
\]
Consider the event,  $\mathcal{A}_0=\{\inf_{y\in \calM} |I_{y,\sigma'}|<1\}$, 
\[
\mathcal{A}_j=\left\{\left\| F_{X,\sigma'}(X_{(j)} + \sigma' \xi) - F_{X,\sigma'}(\exp_{X_{(j)}}(\sigma'\xi_{\mathcal{T}})) \right\| \geq  (d(k+1)\log n)^2 h^2\right\}, 
\]
and $\mathcal{A}=\cup_{j=1}^n \mathcal{A}_j$. 
By Theorem \ref{thm:NWproj}, assuming $\mathcal{A}^c_0$,
$\Prob(\mathcal{A}|X)\leq \sum_j \Prob(\mathcal{A}_j|X)\leq n^{-k}$. So if $\mathcal{A}^c_0$ takes place, 
\[
\sup_{f\in Lip_1}\left|\int \hat{p}_{\text{IDM}}(x)f(x)dx-\E[f(F_{X,\sigma'}(\exp_{X_U}(\sigma'\xi_{\mathcal{T}})) )|X]\right|\leq (d(k+1)\log n)^2h^2+n^{-k}\text{diam}(\calM).
\]
Then Proposition \ref{prop:NWonM} indicates with probability $1-n^{-k}$ 
\[
\sup_{f\in Lip_1}\left|\E[f(F_{X,\sigma'}(\exp_{X_U}(\sigma'\xi_{\mathcal{T}})))|X]-\E[f(\exp_{X_U}(\sigma'\xi_{\mathcal{T}})) |X]\right|\leq \sqrt{(2d+k)\log n}h^2. 
\]

To continue, we note that by Lemma \ref{lem:KDEexp}, 
\[
\sup_{f\in Lip_1}\left|\E[f(\exp_{X_U}(\sigma'\xi_{\mathcal{T}}))|X]-\int f(x) \widehat{p}_{\text{KDE}}(x)dV(x)\right|\leq (2(d+4)\log (2\pi/\sigma))^{\frac{3}{2}}\sigma^2.
\]
Finally, denote 
\[
\mathcal{C}_j=\{|\widehat{p}_{\text{KDE}}(y_j)-p(y_j)|>C h^2\},\quad \mathcal{C}=\cup_{j=1}^{M_\epsilon} \mathcal{C}_j.
\]
Lemma \ref{lem:denominator_F} shows that with probability $$\Prob(\mathcal{C})\leq 
\sum_{j\leq M_\epsilon}\Prob(\mathcal{C}_j)\leq h^{2(4+d)d} h^{-(4+d)d}\leq h^{(4+d)d}$$
Therefore, by Lemma \ref{lem:KDELip}
\begin{align*}
|\widehat{p}_{\text{KDE}}(y)-p(y)|&\leq \min_{j\in [M_\epsilon]}\{|\widehat{p}_{\text{KDE}}(y)-\widehat{p}_{\text{KDE}}(y_j)|+|p(y)-p(y_j)|+|\widehat{p}_{\text{KDE}}(y_j)-p(y_j)|\}\\
&\leq \min_{j\in [M_\epsilon]}\{ (L+(\sigma')^{-d-2})\|y-y_j\|+C h^2\}\leq C h^2. 
\end{align*}
As a consequence, 
\[
\sup_{f\in Lip_1}\left|\int f(x) p(x)dV(x)-\int f(x) \widehat{p}_{\text{KDE}}(x)dV(x)\right|\leq C h^2 \text{diam}(\calM). 
\]
In conclusion, with probability $1-n^{-k}$
\[  
\begin{aligned}
    W_1(\hat{p}_{\text{IDM}}, p_0)&\leq W_1(\hat{p}_{\text{IDM}}, \hat{p}_{\text{KDE}}) + W_1(\hat{p}_{\text{KDE}}, p_0)\\
    &\le C_d d^2k^2(\log n)^2 h^2.
\end{aligned}
\]

Next, we consider the case where $Z_0\sim \mathcal{N}(0,I_D)$. By Lemma \ref{lem:inipatch} below, we can 
couple $Z_{T-h^2}$ and $Z'_{T-h^2}$ exactly with probability $1-n^{-k}$, and so will be $\Zhat_{T}$ and $\Zhat_{T}'$.

If we denote $\phat_{\text{IDM}}'$ to be the distribution of IDM with $Z_0\sim \mathcal{N}(0,I_D)$, then 
\[
\sup_{f:\|f\|_\infty\leq \text{diam}(\calM)} |\int f(x)\phat_{\text{IDM}}(dx)-\int f(x)\phat'_{\text{IDM}}(dx)|\leq n^{-k}\text{diam}(\calM). 
\]
By previous derivation, we know it is ok to assume $f$ to be bounded by $\text{diam}(\calM)$ when considering $W_1$ distance
\[
\sup_{f\in \text{Lip}_1} |\int f(x)\phat_{\text{IDM}}(dx)-\int f(x)p_0(dx)|.
\]

\end{proof}
\begin{lemma}
\label{lem:inipatch}
If we start a diffusion process from $Z_0\sim \mathcal{N}(0,I_D)$ with $T=k(\log n+\log D+\log (\text{diam}(\calM)))$. Then it can be coupled with a diffusion process from $Z'_0\sim \phat_T$ exactly, i.e. $Z_t=Z'_t, t\in [0,T]$, with probability $1-n^{-k}$.  
\end{lemma}
\begin{proof}
    We first note that by convexity of KL divergence,
\begin{align*}
\text{KL}(\phat_T\|\mathcal{N}(0,I_D))&=
\text{KL}(\frac{1}{n}\sum_{i=1}^n \mathcal{N}(\alpha_T X_i, \sigma_T^2 I_D)\|\mathcal{N}(0,I_D))\\
&\leq \frac{1}{N}\sum_{i=1}^n\text{KL}(\mathcal{N}(\alpha_T X_i, \sigma_T^2 I_D)\|\mathcal{N}(0,I_D)). 
\end{align*}
Meanwhile, 
\begin{align*}
\text{KL}(N(\alpha_T X_i, \sigma_T^2 I_D)\|N(0,I_D))
&=\frac12(\alpha^2_T\|X_i\|^2+D(\sigma_T^2-1-\log \sigma_T^2))\\
&\leq \frac12(e^{-2T}\text{diam}(\calM)^2+D e^{-2T}).
\end{align*}
So choosing $T=k(\log n +\log D+\log (\text{diam}(\calM))$, by Pinsker inequality
\[
\text{TV}(\phat_T,N(0,I_D))\leq \frac12\sqrt{\text{KL}(\phat_T\|\mathcal{N}(0,I_D))}\leq n^{-k}.
\]
Let $Z_0$ and $Z'_0$ be two samples from $\phat_T$ and $\mathcal{N}(0,I_D)$ with the optimal coupling. Then $\Prob(Z_0\neq Z'_0)\leq n^{-k}$. Moreover, their DM trajectory will be exactly the same, i.e. $\Prob(\widehat{Z}_t\neq \widehat{Z}'_t)\leq n^{-k}$. 
\end{proof}

\subsection{Analysis for the population DM}
\label{sec:popDM}
\begin{proof}[Proof of Proposition \ref{pro:popDM}]
Since $p_t(x)$ satisfies the forward process, we have the following expression:
\[
p_t(x) = \int_{\calM} p_0(dy) \exp\left( -\frac{\|x - \alpha_t y\|^2}{2 \sigma_t^2} \right),
\]
where $\alpha_t = e^{-t}$ and $\sigma_t^2 = 1 - \alpha_t^2$. 

Now, we calculate the gradient of $\log p_t(x)$:
\[
\begin{aligned}
    \nabla \log p_t(x) &= \frac{1}{p_t(x)} \nabla p_t(x), \\
    -\nabla p_t(x) &= \int_{\calM} \frac{x - \alpha_t y}{\sigma_t^2} \exp\left( -\frac{\|x - \alpha_t y\|^2}{2 \sigma_t^2} \right) p_0(dy).
\end{aligned}
\]

We decompose $x - \alpha_t y$ into terms relative to the manifold projection:
\[
\begin{aligned}
-\nabla \log p_t(x) &= \frac{\int_{\calM} \left( \frac{x - \Pi_{\alpha_t\calM}(x)}{\sigma_t^2} + \frac{\Pi_{\alpha_t\calM}(x) -  \alpha_t y}{\sigma_t^2} \right) \exp\left( -\frac{\|x - \alpha_t y\|^2}{2 \sigma_t^2} \right) p_0(dy)}{\int_{\calM} \exp\left( -\frac{\|x - \alpha_t y\|^2}{2 \sigma_t^2} \right) p_0(dy)}\\
&= \frac{x-\Pi_{\alpha_t\mathcal{ M}}(x)}{\sigma_t^2} - g(x,t),\\
    \text{where}\quad 
    g(x,t)&:=\frac{\int_{\calM}( \frac{\Pi_{\alpha_t\mathcal{M}}(x)-\alpha_t y}{\sigma_t^2}) e^{-\frac{\left\|x-\alpha_t y\right\|^2}{2\sigma_t^2}} p_0(dy)}{\int_{\calM} e^{-\frac{\left\|x-\alpha_t y\right\|^2}{2\sigma_t^2}} p_0(dy)}.
\end{aligned}
\]
 Let $B_{\varepsilon}=\left\{y \in \calM:\|\alpha_t y-x\|^2<\left\|\Pi_{\alpha_t\mathcal{M}}(x)-x\right\|^2+\varepsilon^2\right\}$, which is decreasing set series as $\varepsilon \rightarrow 0$. To handle the complexity of the integrals, we split the domain into a small neighborhood $B_{\varepsilon}$ ( $\epsilon $ to be set later) around $\Prj{x}$ and the remainder of the domain. 

Actually, by definition of $B_{\varepsilon}$, for $y \in \calM \backslash B_{\varepsilon}$,
$$
\left\|x-\alpha_t y\right\|^2 \geq 
\left\|\Pi_{\alpha_t\mathcal{M}}(x)-x\right\|^2+\varepsilon^2$$
So
\begin{align*}
\int_{\calM} e^{-\frac{\left\|x-\alpha_t y\right\|^2}{2\sigma_t^2}} p_0(dy) &\geq
\int_{B_{\sigma_t}} e^{-\frac{\left\|x-\alpha_t y\right\|^2}{2\sigma_t^2}} p_0(dy)\\
&\geq \exp(-\frac{\left\|\Pi_{\alpha_t\mathcal{M}}(x)-x\right\|^2+\sigma_t^2}{2\sigma_t^2})p_{\min}\text{Vol}(B_{\sigma_t})\\
& = C_d   
p_{\min}\sigma_t^{d}\exp(-\frac{\left\|\Pi_{\alpha_t\mathcal{M}}(x)-x\right\|^2}{2\sigma_t^2}),
\end{align*}
where $C_d : = e^{-1/2} \frac{\pi^{d/2}}{\Gamma(d/2+1)}$.

\begin{align*}
& \left\|\frac{\int_{\calM \backslash B_{\varepsilon}}\frac{\Pi_{\alpha_t\mathcal{M}}(x)-\alpha_t y}{\sigma_t^2} e^{-\frac{\left\|x-\alpha_t y\right\|^2}{2\sigma_t^2}} p_0(dy)}{\int_{\calM} e^{-\frac{\left\|x-\alpha_t y\right\|^2}{2\sigma_t^2}} p_0(dy)}\right\|\\
&\leq 
\left\|\frac{\int_{\calM \backslash B_{\varepsilon}}\frac{\Pi_{\alpha_t\mathcal{M}}(x)-\alpha_t y}{\sigma_t^2} e^{\frac{-\left\|x-\Pi_{\alpha_t\calM} x\right\|^2-\epsilon^2}{2\sigma_t^2}} p_0(dy)}{ C_d   
p_{\min}\sigma_t^{d}\exp(-\frac{\left\|\Pi_{\alpha_t\mathcal{M}}(x)-x\right\|^2}{2\sigma_t^2})}\right\| \\
&\leq \frac{max_{y\in \calM \backslash B_\epsilon} \frac{\|\Pi_{\alpha_t\calM}(x)-\alpha_t y\|}{\sigma^2_t}}{C_d\sigma_t^{d} p_{\min}}\exp(-\frac{\epsilon^2}{2\sigma_t^2}) = O(\sigma_t^{4}),
\end{align*}
if we pick $\epsilon=\sigma_t \sqrt{(d+6)\log 1/\sigma_t }.$  

For the part $y\in B_{\varepsilon}$. Denote $z=\Pi_{\calM} (\alpha_t^{-1}x)$, $r=\|\alpha_t^{-1}x-z\|$. Note that 
\[
\|z-y\|\leq \sqrt{r^2+\epsilon^2}+r\leq 2r+\epsilon. 
\]
So with sufficiently small $r_0$, we can find a $u\in \mathcal{T}_z\calM$ with $\|u\|\leq 2r+\epsilon$ so that 
\[
y=\exp_z(u)=z+u+O(\|u\|^2)
\]
Therefore, 
\[
\sqrt{r^2+\epsilon^2}\geq \|y-\alpha_t^{-1}x\|\geq \|u+z-\alpha_t^{-1}x\|-C\|u\|^2=\sqrt{r^2+\|u\|^2}-C(2r+\epsilon)^2.
\]
This leads to 
\[
\|u\|\leq \epsilon+C r^{3/2}+C\epsilon^{3/2},
\]
and for an updated constant $\|z-y\|\leq C\epsilon+Cr^{3/2}$.
\[ 
\begin{aligned}
\frac{\int_{B_{\varepsilon}} \frac{\Pi_{\alpha_t\calM}{x} - \alpha_t y}{\sigma_t^2} \exp\left( -\frac{\|x - \alpha_t y\|^2}{2 \sigma_t^2} \right) p_0(dy)}{\int_{\calM} \exp\left( -\frac{\|x - \alpha_t y\|^2}{2 \sigma_t^2} \right) p_0(dy)}
&\leq max_{y\in B_\epsilon} \alpha_t\frac{\|\Pi_{\calM}(\alpha_t^{-1}x)-y\|}{\sigma^2_t}\\
&=max_{y\in M(\alpha_t^{-1}x, \alpha_t^{-1}\epsilon)} \alpha_t\frac{\|\Pi_{\calM}(\alpha_t^{-1}x)-y\|}{\sigma^2_t}\\
&\leq [C\epsilon+Cdist(\alpha_t^{-1}x,M)^{3/2}]/\sigma^2_t 
\end{aligned}
\]
Moreover, by \cite{leobacher2021existence}, we know $\Pi_{\calM}$ is $C^1$, so  
\[
\|\Pi_{\calM}(x)-\Pi_{\alpha_t\calM}(x)\|
\leq \|\Pi_{\calM}(x)-\Pi_{\calM}(\alpha_t^{-1}x)\|+
(1-\alpha_t)\|\Pi_{\calM}(\alpha_t^{-1}x)\|
\leq O(t\, dist(x,M))
\]
Hence, we have:
\[
\nabla \log p_t(x) =
\frac{x-\Pi_{\calM}(x)}{\sigma_t^2}+
O(1+dist(x,M)^{3/2}/t+\sqrt{d\log (1/t)/t}). 
\]


\end{proof}

\section*{Acknowledgment}
The research of Y.L., Y.Q. and X.T.T. are supported by  Ministry of Education, Singapore, under the Academic Research Fund Tier 1 A-8002956-00-00 and NUS Mathematics department. 
The research of T.M.N. is supported by the Ministry of Education, Singapore, under the Academic Research Fund Tier 1 A-8002040-00-00.

\bibliographystyle{plain}
\bibliography{litdraft}

\end{document}